\documentclass{article}

 \usepackage[preprint]{main}

\usepackage[utf8]{inputenc} 
\usepackage[T1]{fontenc}    
\usepackage{hyperref}       
\usepackage{url}            
\usepackage{booktabs}       
\usepackage{amsfonts}       
\usepackage{nicefrac}       
\usepackage{microtype}      
\usepackage{xcolor}         

\usepackage{tikz}
\usepackage{tikz-cd}
\usetikzlibrary{arrows.meta, positioning, calc}

\usepackage{graphicx}
\usepackage{subcaption}

\usepackage{wrapfig}

\usepackage{amsmath} 
\usepackage{algorithm}
\usepackage{algorithmic}

\usepackage{amsmath}
\usepackage{amssymb}
\usepackage{mathtools}
\usepackage{amsthm}
\usepackage{stmaryrd}

\theoremstyle{plain}
\newtheorem{theorem}{Theorem}[section]
\newtheorem{proposition}[theorem]{Proposition}
\newtheorem{lemma}[theorem]{Lemma}
\newtheorem{corollary}[theorem]{Corollary}
\theoremstyle{definition}
\newtheorem{definition}[theorem]{Definition}
\newtheorem{assumption}[theorem]{Assumption}
\theoremstyle{remark}
\newtheorem{remark}[theorem]{Remark}

\usepackage{multirow}

\newcommand{\Stoch}{\mathsf{Stoch}}
\newcommand{\Meas}{\mathsf{Meas}}

\newcommand{\Tr}{\mathsf{Tr}}
\newcommand{\id}{\mathrm{id}}

\newcommand{\cV}{\mathcal{V}}
\newcommand{\cQ}{\mathcal{Q}}
\newcommand{\E}{\mathbb{E}}
\newcommand{\R}{\mathbb{R}}

\newcommand{\TV}{d_{\mathrm{TV}}}

\newcommand{\CcMet}{\operatorname{CcMet}}

\DeclareMathOperator{\Dirac}{Dirac}
\DeclareMathOperator{\lfp}{lfp}

\title{A Contractive Feedback Semantics for Reinforcement Learning}

\author{%
    Zuyuan Zhang\\
    The George Washington University\\
    \texttt{zuyuan.zhang@gwu.edu}\\
}

\begin{document}

\maketitle

\begin{abstract}
Discounted reinforcement learning is usually presented through Bellman equations on closed Markov decision processes. This paper develops a compositional view: a one-step decision process is treated as an open stochastic component, and infinite-horizon policy evaluation is obtained by closing a contractive feedback loop. The resulting semantics assigns typed Bellman transformers to open components, interprets series and parallel wiring as composition and tensoring of transformers, and interprets feedback as an admissible guarded Banach trace realized by a unique fixed point. This perspective yields three theoretical consequences. First, approximate component equivalence is a contextual congruence for admitted well-typed guarded one-hole contexts: local operator error remains controlled after plugging the component into a surrounding circuit that uses the hole once and whose feedback nodes have certified uniform guardedness. Second, exact and approximate state abstractions become commuting or near-commuting coalgebraic diagrams, giving value-preservation and explicit sup-norm distortion bounds. Third, under monotone $\omega$-continuous contract-transformer semantics, safety, risk, and resource specifications can be represented as quantale-valued contracts, where local inductive bounds lift through wiring and feedback by least-fixed-point reasoning. Its central claim is not that all RL morphisms form a global traced monoidal category, but that discounted Bellman evaluation admits a contractive feedback semantics on the admissible class of guarded circuits.
\end{abstract}

\section{Introduction}
\label{sec:intro}

Reinforcement learning (RL) is commonly formalized as dynamic programming on a closed Markov decision process: fix a policy, write down a Bellman operator, and study the fixed point of that operator \citep{bellman1966dynamic,puterman2014markov,bertsekas1995neuro,sutton1998reinforcement}. This view is effective, but it hides a structural fact. The one-step transition is not intrinsically closed. It is an open component with an input interface, an output interface, a reward channel, and a continuation. The infinite-horizon value appears only after the continuation is fed back into the component. In this paper we make this feedback operation explicit.

The proposed viewpoint is simple but deliberately restricted: \emph{Bellman recursion is a guarded fixed-point trace}. A policy closes the action interface of a one-step stochastic component; discounting makes the continuation map contractive; and policy evaluation is the Banach fixed point obtained by closing the remaining value wire. This converts a familiar analytic fact into a compositional semantics. Series wiring becomes composition of continuation transformers, parallel wiring becomes tensor/product semantics, and feedback becomes a guarded fixed-point trace in the sense of traced monoidal reasoning \citep{joyal1996traced,hasegawa1997recursion,selinger2010survey} and iteration/fixed-point semantics \citep{bloom1993iteration}, specialized here to contractive maps. We use ``trace'' in this guarded sense throughout: only maps whose feedback wire is uniformly contractive are traceable, and all trace laws are invoked only inside this admissible class.

This shift matters because modern decision systems are not monolithic MDPs. They are assembled from modules: perception or representation maps, local controllers, reward shapers, safety monitors, belief-state filters, and off-policy evaluation components. Belief-state filtering is the standard semantic reduction for partially observed control \citep{smallwood1973optimal,kaelbling1998planning}, while off-policy evaluation is naturally phrased as a change of measure between trajectory laws \citep{precup2000eligibility,jiang2016doubly,thomas2016data}. A theory that only describes the final closed MDP has limited language for asking whether a local replacement is valid inside an arbitrary surrounding system. The semantic problem is therefore contextual: if two open components are close at their interface, when are they close after arbitrary composition and feedback?

We answer this by treating open decision components as typed semantic objects. Each component induces a Bellman transformer from output continuations to input values. The transformer type records the direction in which value information flows. Once types are explicit, composition is no longer an informal engineering operation: it is ordinary function composition, with discount factors recording the depth at which each local reward or perturbation enters the feedback loop. This yields quantitative contextual equivalence bounds for the admitted class of well-typed linear one-hole contexts: the hole is used once, side components are fixed, tensoring does not copy the hole, and every feedback node satisfies a uniform guardedness certificate.

The same semantics also clarifies abstraction. A deterministic representation map is a coalgebraic morphism when it commutes with the one-step transition and preserves rewards \citep{rutten2000universal,jacobs2017introduction}. In RL and planning, closely related abstraction, homomorphism, and bisimulation principles have been used to characterize when reduced models preserve decision-relevant behavior \citep{givan2003equivalence,ravindran2004approximate,li2006towards,ferns2004metrics,ferns2011bisimulation}. Exact commutation gives value preservation by an operator-intertwining proof. Approximate commutation gives an explicit sup-norm distortion bound controlled by reward mismatch and total-variation transition mismatch. Because the error is stated at the transformer level, it propagates through admitted larger decision circuits by the contextual congruence theorem.

Finally, the framework separates metric fixed-point reasoning from order-theoretic contract reasoning. Discounted real-valued returns use Banach contraction. Safety, risk, and resource constraints are often better expressed as ordered guarantees. We therefore add a quantale-valued contract layer \citep{lawvere1973metric,eklund2018fundamentals,yetter1990quantales}. In this layer, a circuit is admitted when its open semantics is a monotone $\omega$-continuous contract transformer and its wiring obeys the compilation laws stated later. Under these assumptions, closed loops are interpreted as least fixed points, and local pre-fixed-point certificates lift compositionally through wiring and feedback. This layer connects the paper to classical lattice-theoretic and fixed-point semantics \citep{tarski1955lattice,kleene1952introduction,studer2008proof,cousot1977abstract}.

The contribution is not a new learning algorithm and not an empirical claim. It is a theoretical reframing of RL semantics around contractive feedback. The main results are as follows.

\begin{itemize}
\item \textbf{Guarded traced Bellman semantics.} We formulate open discounted decision components and prove that policy evaluation is the guarded Banach trace of an admissible feedback circuit; no global trace operation on arbitrary maps is assumed.
\item \textbf{Typed compositionality.} We prove that series and parallel wiring correspond to composition and tensoring of Bellman transformers, with explicit discount bookkeeping.
\item \textbf{Contextual equivalence.} We define an operator metric on invariant value balls and prove that approximate equivalence is preserved by admitted well-typed linear contexts generated by composition, tensoring, and guarded trace.
\item \textbf{Coalgebraic abstraction.} We prove exact value preservation under commuting homomorphisms and quantitative value distortion under approximate commutation.
\item \textbf{Quantale contracts.} We prove that inductive safety/risk/resource contracts lift through series, parallel, and feedback wiring whenever the contract-transformer semantics is monotone, $\omega$-continuous, and compatible with the stated wiring laws.
\item \textbf{Extensions and sanity checks.} We show how the same interfaces cover belief-state POMDP lifting, decomposable change of measure under explicit factorization assumptions, time-varying operator tracking, and minimal modular robustness examples.
\end{itemize}

The paper is organized as follows. Section~\ref{sec:motivation} gives the semantic principles. Section~\ref{sec:prelim} fixes the categorical, metric, and order-theoretic interfaces. Section~\ref{sec:oddc} proves Bellman-as-guarded-trace. Section~\ref{sec:compositionality} proves compositionality and contextual congruence. Section~\ref{sec:abstraction} develops exact and approximate abstraction. Section~\ref{sec:contracts} gives the quantale contract layer. Sections~\ref{sec:extensions}--\ref{sec:example} record extensions and explicit robustness examples. All proofs are collected in Appendix~\ref{app:proofs}.

\section{Background and Motivation}
\label{sec:motivation}

The paper is based on three semantic principles.

\textbf{Feedback principle.} Infinite-horizon value is not merely a solution of an equation; it is the result of feeding a continuation output back into a one-step component. Discounting makes this feedback guarded, hence contractive, so the closed-loop value is uniquely defined by Banach's theorem \citep{banach1922operations}; in standard RL notation, this is exactly the contraction argument behind discounted Bellman evaluation \citep{puterman2014markov,sutton1998reinforcement}.

\textbf{Contextual principle.} Components should be compared at their interfaces. If two components induce close typed transformers, then their difference should remain controlled after the component is inserted into a larger decision circuit. This is the decision-theoretic analogue of contextual equivalence in denotational semantics, but with explicit sup-norm sensitivity constants induced by discounted Bellman operators.

\textbf{Contract principle.} Many correctness requirements are upper bounds rather than rewards. Such bounds should compose locally: a guarantee for a subcomponent should lift to a guarantee for a larger wired system. This is naturally expressed by monotone transformers and least fixed points on ordered value spaces \citep{tarski1955lattice,kleene1952introduction,studer2008proof,cousot1977abstract}.

The technical development uses Markov kernels as morphisms in the stochastic layer \citep{frivc2010categorical,panangaden2009labelled,fritz2020synthetic}, standard Borel assumptions to avoid pathologies in conditional probability and disintegration \citep{kallenberg1997foundations}, Banach contraction for discounted real-valued evaluation, and quantale-enriched complete lattices for order-valued contracts \citep{lawvere1973metric,eklund2018fundamentals}. The point of combining these ingredients is not to replace standard MDP theory, but to expose the compositional structure that standard closed-MDP notation suppresses.

\section{Preliminaries}
\label{sec:prelim}

This section fixes the mathematical interfaces used throughout the paper.
Our denotational viewpoint has two layers:
\emph{(i)} the \emph{stochastic layer} in the Markov category $\Stoch$, where one-step dynamics/policies are composable morphisms;
and \emph{(ii)} the \emph{value layer}, where a closed-loop induces a Bellman \emph{transformer} acting on a Banach (or ordered) value space.
We therefore (a) present discounted MDPs in a kernel-theoretic form so that ``one-step dynamics'' becomes a morphism in $\Stoch$;
(b) introduce quantale-enriched value scalars $\cQ$ to uniformly model additive returns and more general ordered guarantees
(risk/resource/safety); and (c) recall traced symmetric monoidal categories, emphasizing the \emph{Banach trace} as a concrete
feedback semantics realized by a unique contractive fixed point. These choices are precisely what makes
Bellman recursion an instance of categorical feedback in Section~\ref{sec:oddc}.

\paragraph{Standing measurability assumptions.}
To keep the development measure-theoretic but lightweight, we assume all spaces
($S,A,R$ and their finite products) are standard Borel measurable spaces.
This ensures the usual closure properties for Markov kernels and expectations used below.

\subsection{Discounted MDPs}
A discounted discrete-time MDP, following the classical Bellman--MDP dynamic-programming formulation \citep{bellman1966dynamic,puterman2014markov,bertsekas1995neuro,sutton1998reinforcement}, is $\mathcal{M}=(S,A,P,r,\gamma)$ with $\gamma\in(0,1)$, where
$S$ and $A$ are measurable spaces,
$P(\cdot\mid s,a)$ is a Markov kernel on $S$ (equivalently $P:S\times A\to \Delta(S)$),
and $r:S\times A\to\R$ is a bounded measurable reward.
A stationary policy is a Markov kernel $\pi(\cdot\mid s)$ on $A$ (equivalently $\pi:S\to\Delta(A)$).
Given $s_0=s$, trajectories are generated by
$a_t\sim \pi(\cdot\mid s_t)$ and $s_{t+1}\sim P(\cdot\mid s_t,a_t)$.
The (discounted) value is
$
V^\pi(s)=\E^\pi\Big[\sum_{t\ge 0}\gamma^t r(s_t,a_t)\,\Big|\,s_0=s\Big].
$

\begin{lemma}[Boundedness of discounted values]
\label{lem:value-bounded}
If $|r(s,a)|\le R_{\max}$ for all $(s,a)$, then for every stationary policy $\pi$,
$V^\pi$ is bounded and satisfies $\|V^\pi\|_\infty \le V_{\max}:=\frac{R_{\max}}{1-\gamma}$.
\end{lemma}

\subsection{The Markov category \texorpdfstring{$\Stoch$}{Stoch}}
\label{sec:stoch}
Let $\Meas$ denote measurable spaces and measurable maps, and write $\Delta(Y)$ for probability
measures on $Y$. This kernel-based presentation follows the standard measure-theoretic treatment of Markov processes and its categorical formulation via stochastic/Markov categories \citep{frivc2010categorical,panangaden2009labelled,fritz2020synthetic,kallenberg1997foundations}.

\begin{definition}[Markov kernel]
A Markov kernel $k:X \to \Delta(Y)$ assigns to each $x\in X$ a probability measure $k(\cdot\mid x)$ on $Y$
such that $x\mapsto k(C\mid x)$ is measurable for every measurable $C\subseteq Y$.
\end{definition}

\paragraph{Deterministic maps as kernels.}
Every measurable map $f:X\to Y$ induces a \emph{deterministic} kernel
$\Dirac(f):X\to \Delta(Y)$ defined by $\Dirac(f)(\cdot\mid x):=\delta_{f(x)}$.
We use this embedding implicitly whenever a deterministic wiring map appears inside $\Stoch$ compositions.

\begin{definition}[$\Stoch$]
$\Stoch$ has measurable spaces as objects and Markov kernels as morphisms.
Composition is Chapman--Kolmogorov:
\[
(\ell \circ k)(C \mid x) := \int_Y \ell(C \mid y)\, k(dy \mid x).
\]
The symmetric monoidal product $\otimes$ is the cartesian product of measurable spaces,
with unit object $I=\{\star\}$.
\end{definition}

\paragraph{Closed-loop pairing kernel.}
Given a policy kernel $\pi:S\to \Delta(A)$, the canonical kernel
$\langle \id_S,\pi\rangle:S\to \Delta(S\otimes A)$ is
$
\langle \id_S,\pi\rangle(d(s,a)\mid s_0)
\;:=\;
\delta_{s_0}(ds)\,\pi(da\mid s_0).
$
Intuitively, it \emph{copies} the current state to the first output wire and samples an action from $\pi$ on the second wire.

\noindent
An MDP transition $P(\cdot\mid s,a)$ is exactly a $\Stoch$-morphism $P:S\otimes A\to S$,
and a policy is a $\Stoch$-morphism $\pi:S\to A$.
Hence one-step closed-loop dynamics becomes
$
P^\pi := P\circ \langle \id_S,\pi\rangle : S\to S
\quad\text{in }\Stoch,
$
which is the categorical backbone of our compositional wiring results later.

\subsection{Quantale-enriched value spaces}
\label{sec:quantale}

Additive discounted return is only one instance of a value algebra.
Safety, risk, and resource budgets are naturally phrased as \emph{ordered guarantees}
(e.g., ``value is below a specification bound'').
Quantales provide a minimal algebraic interface that supports (i) composition of guarantees via a monoidal product
and (ii) joins capturing choice/nondeterminism \citep{lawvere1973metric,eklund2018fundamentals,yetter1990quantales}; Section~\ref{sec:contracts} builds the contract layer on this interface.

\begin{definition}[Quantale]
A quantale is $(\cQ,\le,\bigvee,\otimes,e)$ where $(\cQ,\le)$ is a complete lattice,
$(\cQ,\otimes,e)$ is a monoid, and $\otimes$ distributes over arbitrary joins:
\[
q\otimes\Big(\bigvee_i q_i\Big) \;=\; \bigvee_i (q\otimes q_i),
\qquad
\Big(\bigvee_i q_i\Big)\otimes q \;=\; \bigvee_i (q_i\otimes q).
\]
\end{definition}

\begin{definition}[Discount endomorphism]
A discount is a monotone map $\delta:\cQ\to\cQ$ that preserves joins and is a monoid homomorphism:
\[
\delta\Big(\bigvee_i q_i\Big)=\bigvee_i\delta(q_i),
\qquad
\delta(q\otimes q')=\delta(q)\otimes\delta(q'),
\qquad
\delta(e)=e.
\]
\end{definition}

For ordered contracts, our running quantitative example is
\[
\cQ=[0,\infty],\qquad \bigvee=\sup,\qquad q\otimes q'=q+q',\qquad e=0,
\]
with the convention $q+\infty=\infty$ and discount endomorphism $\delta(q)=\gamma q$ for
$\gamma\in(0,1)$, where $\delta(\infty)=\infty$.
This gives a well-defined complete quantale for nonnegative risk, cost, violation, or resource bounds.

Real-valued discounted return is handled separately.  It is not modeled as the complete quantale
$(\overline{\R},+,0)$, since $+\infty+(-\infty)$ is not a well-defined monoid operation.  When we analyze
standard rewards and Banach fixed points (Section~\ref{sec:oddc}), we instead work on the Banach space
of bounded real-valued functions
\[
\mathcal{B}(S):=\{V:S\to\R \text{ measurable and bounded}\},
\qquad
\|V\|_\infty:=\sup_{s\in S}|V(s)|.
\]
Lemma~\ref{lem:value-bounded} ensures $V^\pi\in \mathcal{B}(S)$ under bounded rewards.

\paragraph{Contract value space.}
For a state space $S$ and a fixed contract quantale $\cQ$, define the pointwise ordered function space
\[
\cV(S):=\cQ^S,\qquad
(V\le W) \iff (\forall s\in S,\; V(s)\le W(s)).
\]
The notation $\cV(S)$ is used only for the order-valued contract layer; the Banach layer uses
$\mathcal{B}(S)$.

\begin{lemma}[Banachness of bounded measurable functions]
\label{lem:BS-banach}
$\big(\mathcal{B}(S),\|\cdot\|_\infty\big)$ is a Banach space.
\end{lemma}

\subsection{Traces and guarded fixed-point feedback}
\label{sec:trace}

\begin{definition}[Guarded Banach trace]
\label{def:banach-trace}
Let $X,Y,Z$ be complete metric spaces. A map
$
f:X\times Z\to Y\times Z,\qquad f=(f_Y,f_Z),
$
is \emph{$Z$-guarded} if its feedback component is uniformly contractive in the feedback variable: there exists
$c\in[0,1)$ such that
$
d_Z\big(f_Z(x,z),f_Z(x,z')\big)\le c\,d_Z(z,z')
\quad
\forall x\in X,\ z,z'\in Z.
$
For a $Z$-guarded map $f$ and each $x\in X$, Banach's fixed-point theorem gives a unique point
$z_x\in Z$ satisfying
$
z_x=f_Z(x,z_x).
$
The guarded Banach trace is the partial operation
$
\Tr^Z_{X,Y}(f)(x):=f_Y(x,z_x).
$
It is defined only for guarded maps. Thus $\Tr$ in this paper denotes a guarded fixed-point operation, not a total trace on all maps.
\end{definition}

\noindent
\emph{Scope of the trace terminology.}
We use the word ``trace'' in the restricted fixed-point sense above. The Bellman transformers considered below are not claimed to form a global traced symmetric monoidal category. Whenever we use trace-style identities such as naturality, superposing, or vanishing, the identity is invoked only for expressions whose relevant feedback coordinates are guarded on the stated invariant balls. This is the only trace structure needed for discounted Bellman evaluation: the discount factor supplies the contraction, and Banach's theorem supplies the unique closed-loop value.

\noindent
\emph{Interpretation.}
Tracing out $Z$ closes a feedback loop on $Z$ by solving a unique contractive fixed point. In Section~\ref{sec:oddc}, $Z$ is a value-function space and the contraction is exactly the discount factor. Appendix~\ref{app:banach-trace-axioms} records the fixed-point facts and the guarded trace laws used in the sequel.

\section{Open Decision Components and Traced Bellman Semantics}
\label{sec:oddc}

Building on Section~\ref{sec:prelim}, we package a \emph{one-step} decision process as an \emph{open} stochastic component
(a morphism in $\Stoch$) and show that the usual infinite-horizon value semantics is obtained by \emph{closing a feedback loop}.
Concretely, we (i) define open discounted decision components (ODDCs) as kernels producing next state and one-step reward;
(ii) derive the induced Bellman backup operator on a Banach value space; and (iii) show that the Bellman fixed point is
exactly the \emph{trace} of a feedback circuit, realized as a contractive fixed point (Banach trace).
This section focuses on the standard \emph{additive discounted return} so the feedback semantics is fully explicit;
general quantale contracts are developed later.

The \emph{component} lives in $\Stoch$ (kernel-level wiring).
Fixing a policy turns the component into a \emph{Bellman transformer} $\mathcal{T}$ acting on $\mathcal{B}(S)$.
The feedback/trace argument is carried out in $\CcMet_{\mathrm{ctr}}$ on the value space:
this is where contractive fixed points provide a canonical notion of ``closing the loop.''

\subsection{Open discounted decision components}
\label{sec:oddc-def}

Throughout this section, $S$ and $A$ are measurable spaces (state/action),
$R$ is a measurable space of one-step reward signals, and
$K:S\otimes A \to S\otimes R$ is a Markov kernel in $\Stoch$ that outputs a pair $(s',r)$ given $(s,a)$.
A policy is a kernel $\pi:S\to A$ in $\Stoch$.

\begin{definition}[Open discounted decision component (ODDC)]
An open discounted decision component is $\mathsf{M}=(S,A,R,K,\rho,\gamma)$ where
$K:(S\otimes A)\to (S\otimes R)$ is a kernel in $\Stoch$,
$\rho:R\to\R$ is a measurable scalarization (e.g., $\rho(r)=r$ when $R=\R$),
and $\gamma\in(0,1)$ is the discount factor.
\end{definition}

\paragraph{Closing the loop with a policy.}
Given $\pi:S\to A$, define the closed-loop one-step kernel
$
K^\pi := K \circ \langle \id_S,\pi\rangle : S \to (S\otimes R),
$
so that $(s',r)\sim K^\pi(\cdot\mid s)$ corresponds to
$a\sim\pi(\cdot\mid s)$ followed by $(s',r)\sim K(\cdot\mid s,a)$.

\paragraph{Value space.}
Let $\mathcal{B}(S)$ denote the Banach space of bounded measurable functions $V:S\to\R$
equipped with $\|\cdot\|_\infty$ (Lemma~\ref{lem:BS-banach}).

\paragraph{Bellman backup (additive discounted return).}
Define $\mathcal{T}_{\mathsf{M},\pi}:\mathcal{B}(S)\to\mathcal{B}(S)$ by
\begin{equation}
\label{eq:bellman-additive}
(\mathcal{T}_{\mathsf{M},\pi}V)(s)
:=
\int_{S\otimes R}\Bigl(\rho(r) + \gamma V(s')\Bigr)\,K^\pi(d(s',r)\mid s).
\end{equation}
This is the standard expectation $\E[\rho(r)+\gamma V(s')]$, written in kernel notation so that later wiring
statements can be expressed purely through $\Stoch$ composition/tensoring.

\begin{lemma}[Well-definedness and boundedness]
\label{lem:bellman-well-defined}
Assume the scalarized one-step reward is uniformly bounded:
$
|\rho(r)|\le R_{\max}\qquad \text{for all } r\in R .
$
Then $\mathcal{T}_{\mathsf{M},\pi}$ maps $\mathcal{B}(S)$ to itself and
$
\|\mathcal{T}_{\mathsf{M},\pi}V\|_\infty \;\le\; R_{\max} + \gamma \|V\|_\infty.
$
In particular, $\mathcal{T}_{\mathsf{M},\pi}$ is a self-map on the closed ball
$\{V:\|V\|_\infty\le V_{\max}\}$ with $V_{\max}:=\frac{R_{\max}}{1-\gamma}$.
\end{lemma}

\begin{lemma}[Monotonicity and affine Lipschitz structure]
\label{lem:bellman-monotone}
Under Lemma~\ref{lem:bellman-well-defined}, $\mathcal{T}_{\mathsf{M},\pi}$ is monotone w.r.t.\ the pointwise order:
if $V\le W$ then $\mathcal{T}_{\mathsf{M},\pi}V\le \mathcal{T}_{\mathsf{M},\pi}W$.
Moreover, for any constant $c\in\R$ (identified with the constant function $s\mapsto c$),
\[
\mathcal{T}_{\mathsf{M},\pi}(V+c)=\mathcal{T}_{\mathsf{M},\pi}(V)+\gamma c,
\qquad
\|\mathcal{T}_{\mathsf{M},\pi}V-\mathcal{T}_{\mathsf{M},\pi}W\|_\infty\le \gamma\|V-W\|_\infty.
\]
\end{lemma}

\begin{proposition}[Trajectory semantics agrees with fixed-point semantics]
\label{prop:traj-fp}
Let $(s_{t+1},r_t)\sim K^\pi(\cdot\mid s_t)$ and assume the scalarized reward satisfies
$|\rho(r)|\le R_{\max}$ for all $r\in R$.
Define the closed-loop trajectory value
$
V^\pi(s)
:=
\E\!\left[
\sum_{t\geq 0}\gamma^t \rho(r_t)
\,\middle|\, s_0=s
\right].
$
Then $V^\pi\in\mathcal{B}(S)$ and is the unique fixed point of
$\mathcal{T}_{\mathsf{M},\pi}$ on $\mathcal{B}(S)$.
\end{proposition}

\subsection{Why Bellman equals trace}
\label{sec:bellman-trace-explain}

Equation~\eqref{eq:bellman-additive} describes a \emph{one-step} continuation update:
given a continuation $V$ for the next state, $\mathcal{T}_{\mathsf{M},\pi}V$ is the resulting current-state value.
Infinite-horizon evaluation closes the loop by demanding consistency of the continuation with its own update:
$V=\mathcal{T}_{\mathsf{M},\pi}V$.
A traced symmetric monoidal category axiomatizes exactly this ``wire the output back to the input'' operation.
When the feedback variable is contractive (here via $\gamma\in(0,1)$ under $\|\cdot\|_\infty$),
the trace is realized concretely by a unique fixed point (Banach trace).

\begin{lemma}[Contraction of the Bellman backup]
\label{lem:bellman-contraction}
Under the assumptions of Lemma~\ref{lem:bellman-well-defined},
for all $V,W\in\mathcal{B}(S)$,
\[
\|\mathcal{T}_{\mathsf{M},\pi}V - \mathcal{T}_{\mathsf{M},\pi}W\|_\infty
\;\le\; \gamma\,\|V-W\|_\infty.
\]
Hence $\mathcal{T}_{\mathsf{M},\pi}$ is a $\gamma$-contraction on $\big(\mathcal{B}(S),\|\cdot\|_\infty\big)$.
\end{lemma}

\begin{theorem}[Value iteration converges for policy evaluation]
\label{thm:policy-eval-vi}
Under Lemma~\ref{lem:bellman-well-defined}, $\mathcal{T}_{\mathsf{M},\pi}$ has a unique fixed point $V^\pi$.
Moreover, for any initial $V_0\in\mathcal{B}(S)$, the iterates $V_{k+1}:=\mathcal{T}_{\mathsf{M},\pi}V_k$ satisfy
\[
\|V_k - V^\pi\|_\infty \le \gamma^k \|V_0 - V^\pi\|_\infty,
\]
so the convergence is geometric with rate $\gamma$.
\end{theorem}

\begin{theorem}[Bellman recursion as a guarded Banach trace]
\label{thm:bellman-trace}
Assume the conditions of Lemma~\ref{lem:bellman-well-defined}.
Define the feedback map
\[
f_\pi:\ I\times \mathcal{B}(S)\ \longrightarrow\ \mathcal{B}(S)\times \mathcal{B}(S),
\qquad
f_\pi(\star,V):=\big(\mathcal{T}_{\mathsf{M},\pi}V,\ \mathcal{T}_{\mathsf{M},\pi}V\big).
\]
The feedback component of $f_\pi$ is $\gamma$-contractive by Lemma~\ref{lem:bellman-contraction}, so $f_\pi$ is admissible for the guarded Banach trace of Definition~\ref{def:banach-trace}. Its trace is well-defined and satisfies
$
\Tr^{\mathcal{B}(S)}_{I,\mathcal{B}(S)}(f_\pi)(\star)=V^\pi,
$
where $V^\pi$ is the unique Bellman fixed point (equivalently, the discounted return in
Proposition~\ref{prop:traj-fp}).
\end{theorem}

This theorem is a representation theorem for standard discounted policy evaluation. It does not assert that arbitrary stochastic components or arbitrary Bellman transformers carry a total trace. It says that the usual Bellman fixed point is exactly the result of closing the continuation wire of the one-step Bellman transformer, provided the feedback loop is guarded by discounting.

By Definition~\ref{def:banach-trace}, $\Tr^{\mathcal{B}(S)}_{I,\mathcal{B}(S)}(f_\pi)(\star)$
solves the unique fixed point of the feedback map
$V\mapsto \pi_Z f_\pi(\star,V)=\mathcal{T}_{\mathsf{M},\pi}V$,
and outputs the corresponding $Y$-component, which equals $\mathcal{T}_{\mathsf{M},\pi}V$.
Thus the traced value is exactly the Bellman fixed point.

\section{Compositionality and Contextual Equivalence}
\label{sec:compositionality}

Section~\ref{sec:oddc} established the \emph{value-layer} semantics of an open one-step component:
after fixing a policy, a kernel induces a \emph{Bellman transformer}
$\mathcal{T}:\mathcal{B}(\text{out-state})\to\mathcal{B}(\text{in-state})$, and closing feedback corresponds to a trace.
This section proves the next pillar: \emph{compositionality}.
We first show that wiring open components in string diagrams (series/parallel) is reflected \emph{exactly}
by structural operations on their induced transformers, with the correct discount bookkeeping.
We then formalize \emph{contextual equivalence}: if two open components are close (in an operator metric),
they remain close after being plugged into any admissible well-formed circuit context generated by $\circ,\otimes,\Tr$.
Crucially, we make the stability constants \emph{explicitly computable} from the context structure, rather than existential.

\subsection{String-diagram wiring: series and parallel}
\label{sec:wiring}

In $\Stoch$, the monoidal product is cartesian product.
Thus ``parallel'' wiring is modeled by a product kernel, while ``series'' wiring is modeled by kernel composition
through an interface object.
Throughout this subsection, \emph{series} means temporal serial wiring: the first component advances one micro-step and the second component advances a subsequent micro-step.  Consequently, rewards or errors generated by the second component appear one discount factor deeper.  If two engineering modules are merely internal subroutines of the same environment step, they should first be fused into a single primitive one-step transformer rather than assigned an additional factor of $\gamma$ by the series rule.
When composing rewards, the stochastic layer naturally aggregates one-step signals as a pair
$(r_1,r_2)\in R_1\otimes R_2 \cong R_1\times R_2$, which is then mapped to a scalar at the value layer
(e.g., by an additive scalarization $\rho_\otimes(r_1,r_2)=\rho_1(r_1)+\rho_2(r_2)$).

\begin{figure}[t]
\centering
\resizebox{\linewidth}{!}{%
\begin{tikzpicture}[x=1.25cm,y=1cm,baseline=(current bounding box.center)]
\node[draw,rounded corners,minimum width=1.7cm,minimum height=0.8cm] (M1) at (0,0) {$\mathsf{M}_1$};
\node[draw,rounded corners,minimum width=1.7cm,minimum height=0.8cm] (M2) at (2.6,0) {$\mathsf{M}_2$};
\draw[-{Latex}] (-1.1,0.25) -- (M1.west |- 0,0.25);
\draw[-{Latex}] (M1.east |- 0,0.25) -- (M2.west |- 0,0.25);
\draw[-{Latex}] (M2.east |- 0,0.25) -- (3.7,0.25);
\node at (-1.05,0.55) {\scriptsize $S\otimes A$};
\node at (3.2,0.55) {\scriptsize $S\otimes (R_1\otimes R_2)$};

\draw[-{Latex}] (-1.1,-0.25) -- (M1.west |- 0,-0.25);
\draw[-{Latex}] (M1.east |- 0,-0.25) -- (M2.west |- 0,-0.25);
\draw[-{Latex}] (M2.east |- 0,-0.25) -- (3.7,-0.25);

\node at (1.3,-0.95) {\scriptsize (a) Series: wiring by composition};

\begin{scope}[xshift=7cm] 
\node[draw,rounded corners,minimum width=1.7cm,minimum height=0.8cm] (P1) at (0,0) {$\mathsf{M}_1$};
\node[draw,rounded corners,minimum width=1.7cm,minimum height=0.8cm] (P2) at (2.6,0) {$\mathsf{M}_2$};

\draw[-{Latex}] (-1.1,0.25) -- (P1.west |- 0,0.25);
\draw[-{Latex}] (P1.east |- 0,0.25) -- (3.7,0.25);
\node at (-1.05,0.55) {\scriptsize $S_1\otimes A_1$};
\node at (3.7,0.55) {\scriptsize $S_1\otimes R_1$};

\draw[-{Latex}] (-1.1,-0.25) -- (P2.west |- 0,-0.25);
\draw[-{Latex}] (P2.east |- 0,-0.25) -- (3.7,-0.25);
\node at (-1.05,-0.55) {\scriptsize $S_2\otimes A_2$};
\node at (3.7,-0.55) {\scriptsize $S_2\otimes R_2$};

\node at (1.3,-0.9) {\scriptsize (b) Parallel: wiring by tensor product};

\end{scope}

\end{tikzpicture}%
}
\caption{Series and parallel wiring. Series rewards are scalarized at the value layer.}
\label{fig:string-diagrams}
\vspace{-0.2cm}
\end{figure}
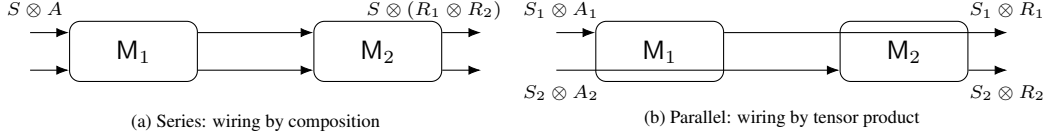

\subsection{Bellman transformers with explicit types}
\label{sec:typed-bellman}

A component with kernel $K:X\otimes A\to Y\otimes R$, scalarization $\rho:R\to\R$, and discount $\gamma\in(0,1)$
induces (after fixing a policy $\pi_X:X\to A$) a closed-loop kernel
$K^{\pi_X}:=K\circ\langle \id_X,\pi_X\rangle:X\to Y\otimes R$.
The induced Bellman \emph{transformer} is the typed map
\[
\mathcal{T}_{K,\pi_X}:\mathcal{B}(Y)\to \mathcal{B}(X),
\qquad
(\mathcal{T}_{K,\pi_X}V)(x):=\int_{Y\otimes R}\big(\rho(r)+\gamma V(y)\big)\,K^{\pi_X}(d(y,r)\mid x).
\]
The key point is the \emph{continuation direction}: $V$ lives on the \emph{output state} $Y$ and is pulled back to $X$.
This is exactly why series wiring becomes ordinary function composition.

\subsection{Compositional Bellman semantics}
\label{sec:comp-bellman}

Consider two temporal micro-step components connected through an interface $U$:
\[
K_1:S\otimes A_1\to U\otimes R_1,
\qquad
K_2:U\otimes A_2\to S\otimes R_2,
\]
with scalarizations $\rho_1:R_1\to\R$ and $\rho_2:R_2\to\R$.
Let $\pi_S:S\to \Delta(A_1)$ and $\pi_U:U\to \Delta(A_2)$ be policies for the two micro-steps.
The two action interfaces are allowed to differ; if a model supplies a deterministic interface controller or a transported policy, it is included as part of $\pi_U$.

Define the micro-step transformers
\[
\mathcal{T}_1:=\mathcal{T}_{K_1,\pi_S}:\mathcal{B}(U)\to\mathcal{B}(S),
\qquad
\mathcal{T}_2:=\mathcal{T}_{K_2,\pi_U}:\mathcal{B}(S)\to\mathcal{B}(U).
\]
Their series wiring yields a macro-step transformer on $\mathcal{B}(S)$:
$
\mathcal{T}_{\mathrm{series}} := \mathcal{T}_1\circ \mathcal{T}_2:\mathcal{B}(S)\to\mathcal{B}(S).
$
Unfolding the definitions makes the discount bookkeeping explicit:
the second-stage reward is discounted once more, and the continuation is discounted twice.

\begin{theorem}[Compositionality under series wiring]
\label{thm:compositionality-series}
Assume bounded scalar rewards $|\rho_i(r_i)|\le R_{\max,i}$ and common discount $\gamma\in(0,1)$.
For series wiring through $U$ as above, for every $V\in\mathcal{B}(S)$ and $s\in S$,
\[
(\mathcal{T}_{\mathrm{series}}V)(s)
=
\E\Big[\rho_1(r_1) + \gamma\,\rho_2(r_2) + \gamma^2 V(s_2)\ \Big|\ s_0=s\Big],
\]
where $(u_1,r_1)\sim K_1^{\pi_S}(\cdot\mid s_0)$ and $(s_2,r_2)\sim K_2^{\pi_U}(\cdot\mid u_1)$.
Equivalently,
$
\mathcal{T}_{\mathrm{series}}=\mathcal{T}_{K_1,\pi_S}\circ \mathcal{T}_{K_2,\pi_U}.
$
Moreover, $\mathcal{T}_{\mathrm{series}}$ is a $\gamma^2$-contraction on $\big(\mathcal{B}(S),\|\cdot\|_\infty\big)$.
\end{theorem}

For parallel wiring, the stochastic layer forms the product kernel $K_\otimes:=K_1\otimes K_2$ in $\Stoch$
together with the product policy $\pi_\otimes:=\pi_1\otimes \pi_2$.
The resulting Bellman operator acts on $\mathcal{B}(S_1\otimes S_2)$ and uses the additive scalarization
$\rho_\otimes(r_1,r_2):=\rho_1(r_1)+\rho_2(r_2)$.
To connect this to the values of the two subsystems, we use the standard separable embedding below.

\begin{definition}[Separable embedding]
\label{def:separable-embed}
For $V_1\in\mathcal{B}(S_1)$ and $V_2\in\mathcal{B}(S_2)$, define
\[
(V_1\oplus V_2)(s_1,s_2)\;:=\;V_1(s_1)+V_2(s_2),
\qquad (s_1,s_2)\in S_1\otimes S_2.
\]
\end{definition}

\begin{theorem}[Additive factorization under parallel wiring]
\label{thm:compositionality-parallel}
Assume the two subsystems are independent in the sense that
$K_\otimes = K_1\otimes K_2$ and $\pi_\otimes=\pi_1\otimes \pi_2$, and use the additive scalarization
$\rho_\otimes(r_1,r_2)=\rho_1(r_1)+\rho_2(r_2)$ with a common discount $\gamma\in(0,1)$.
Let $\mathcal{T}_\otimes$ denote the Bellman operator induced by $(K_\otimes,\pi_\otimes)$ on $\mathcal{B}(S_1\otimes S_2)$,
and let $\mathcal{T}_i$ be the Bellman operators induced by $(K_i,\pi_i)$ on $\mathcal{B}(S_i)$.

Then the separable subspace is invariant:
\[
\mathcal{T}_\otimes(V_1\oplus V_2) \;=\; (\mathcal{T}_1 V_1)\oplus(\mathcal{T}_2 V_2)
\quad\text{for all }V_1,V_2.
\]
Consequently, if $V^{\pi_i}$ is the unique fixed point of $\mathcal{T}_i$ for $i\in\{1,2\}$,
then the unique fixed point of $\mathcal{T}_\otimes$ satisfies
$
V^{\pi_\otimes}(s_1,s_2)=V^{\pi_1}(s_1)+V^{\pi_2}(s_2).
$
\end{theorem}

Theorems~\ref{thm:compositionality-series}--\ref{thm:compositionality-parallel} state that our denotational semantics
is \emph{structural}: series wiring corresponds to ordinary composition of continuation transformers (with explicit discount depth),
while parallel wiring preserves an additive/separable value structure under independence.
This structurality is the foundation for contextual reasoning: once the semantic constructors are Lipschitz and every feedback node carries a guardedness certificate, approximate equivalences propagate through admitted linear circuit contexts. Here ``linear'' means that the hole is used exactly once; contexts that duplicate, merge, or nonlinearly inspect the hole require separate sensitivity bounds.

\subsection{Contextual equivalence as congruence}
\label{sec:contextual}

To avoid the trivial $\sup_V=\infty$ issue, discrepancies are always measured on typed invariant balls.
For a state space $X$ and radius $M>0$, write
\[
\mathbb{B}_M(X):=\{V\in\mathcal{B}(X):\|V\|_\infty\le M\}.
\]
For two typed transformers
$T,T':\mathbb{B}_{M_Y}(Y)\to\mathbb{B}_{M_X}(X)$ define
\begin{equation}
\label{eq:typed-operator-metric}
d_\infty^{X\leftarrow Y}(T,T')
:=\sup_{V\in\mathbb{B}_{M_Y}(Y)}\|TV-T'V\|_{\infty,X}.
\end{equation}
When $X=Y=S$ and $M_X=M_Y=M$, this reduces to the closed-operator metric $d_\infty^{(S,M)}$.
All Bellman transformers below are assumed to come with such invariant balls; by Lemma~\ref{lem:bellman-well-defined}, this is automatic after choosing a radius at least as large as the corresponding value bound.

\begin{lemma}[Fixed-point stability for contractions]
\label{lem:fp-stability}
Let $T,T':\mathbb{B}_M(S)\to\mathbb{B}_M(S)$ be $\kappa$-contractions in $\|\cdot\|_\infty$ for some $\kappa\in(0,1)$,
with unique fixed points $V^\star$ and ${V^\star}'$.
Then
\[
\|V^\star-{V^\star}'\|_\infty \le \frac{1}{1-\kappa}\, d_\infty^{(S,M)}(T,T').
\]
\end{lemma}

\begin{definition}[Certified admitted linear guarded context]
\label{def:admitted-linear-context}
Fix a hole type
$
\mathbb{B}_{M_Y}(Y)\longrightarrow \mathbb{B}_{M_X}(X)
$
in the continuation direction. A \emph{linear one-hole context} $\mathcal C[\cdot]$ is a well-typed expression generated by the grammar
\[
\mathcal C
::=
[\cdot]
\mid T_0\circ \mathcal C
\mid \mathcal C\circ T_0
\mid \mathcal C\otimes T_0
\mid T_0\otimes \mathcal C
\mid \Tr^Z(\mathcal D),
\]
where $T_0$ ranges over fixed side transformers and $\mathcal D$ is itself a one-hole pre-trace expression of product type
$X\times Z\to Y\times Z$ on the relevant invariant balls.

The word \emph{linear} means syntactic single-use of the hole: no diagonal copying, merging of two copies, deletion followed by reintroduction, or nonlinear inspection of the hole is allowed. Tensor products use the product sup metric. If a later scalarization adds, maximizes, thresholds, or otherwise post-processes tensor outputs, that post-processing must be represented as an additional fixed Lipschitz transformer $T_0$ in the grammar.

A context is \emph{admitted} if every subexpression carries a certificate consisting of:
\begin{enumerate}
\item source and target invariant balls;
\item a Lipschitz constant for every fixed side transformer;
\item for each trace node, uniform constants
$\alpha_Z<1,\eta_Z,\beta_Z,a_X$ satisfying the inequalities in Lemma~\ref{lem:context-recursion};
\item a structurally computed perturbation gain $L(\mathcal C)<\infty$;
\item for a closed context, a contraction modulus $\kappa(\mathcal C)<1$ on the final invariant ball.
\end{enumerate}
The same certificate must be valid for both plugged components when two circuits are compared.
\end{definition}

\begin{definition}[Context gain and context contraction modulus]
\label{def:context-gain}
For a certified admitted context $\mathcal C[\cdot]$, the \emph{context gain} $L(\mathcal C)$ is the finite perturbation-amplification constant generated by the structural rules of Lemma~\ref{lem:context-recursion}. Thus, for all compatible components $\mathsf M,\mathsf N$ with the same hole type,
\[
d_\infty^{(S,M)}\!\big(\mathcal{T}_{\mathcal{C}[\mathsf M]},\mathcal{T}_{\mathcal{C}[\mathsf N]}\big)
\le
L(\mathcal C)\,
d_\infty^{X\leftarrow Y}(\mathcal{T}_{\mathsf M},\mathcal{T}_{\mathsf N}).
\]
We write $\kappa(\mathcal C)$ for a certified contraction modulus of the induced closed operator on its invariant ball.
Thus $L(\mathcal C)$ measures perturbation amplification, whereas $\kappa(\mathcal C)$ measures closed-loop feedback sensitivity.
\end{definition}

\begin{lemma}[Structural recursion for $L(\mathcal C)$ and $\kappa(\mathcal C)$]
\label{lem:context-recursion}
Assume every primitive transformer in the context maps its typed invariant ball to the next typed invariant ball and has a known Lipschitz constant. For each trace subexpression, write the pre-trace feedback map as
\[
F:X\times Z\to Y\times Z,
\qquad F(x,z)=(F_Y(x,z),F_Z(x,z)).
\]
Assume that, on the relevant balls,
\begin{align*}
 d_Z(F_Z(x,z),F_Z(x,z'))&\le \alpha_Z d_Z(z,z'), &&\alpha_Z<1,\\
 d_Z(F_Z(x,z),F_Z(x',z))&\le \eta_Z d_X(x,x'),\\
 d_Y(F_Y(x,z),F_Y(x,z'))&\le \beta_Z d_Z(z,z'),\\
 d_Y(F_Y(x,z),F_Y(x',z))&\le a_X d_X(x,x').
\end{align*}
Then the traced map $\Tr^Z(F):X\to Y$ is Lipschitz with
$
\operatorname{Lip}(\Tr^Z(F))
\le
a_X+\frac{\beta_Z\eta_Z}{1-\alpha_Z}.
$
A traced closed-loop subexpression is admitted only when the resulting external modulus is strictly smaller than one.

The perturbation gain is certified by the following structural rules:
\begin{itemize}
\item \textbf{Hole.} $L([\cdot])=1$.
\item \textbf{Left composition.}
$
L(T_0\circ \mathcal C)\le \operatorname{Lip}(T_0)L(\mathcal C).
$
\item \textbf{Right composition.}
$
L(\mathcal C\circ T_0)\le L(\mathcal C),
$
provided $T_0$ maps the source ball into the discrepancy ball on which the hole-side transformer is compared.
\item \textbf{Tensor with a fixed side context.}
$
L(\mathcal C\otimes T_0)\le L(\mathcal C),
\qquad
L(T_0\otimes \mathcal C)\le L(\mathcal C),
$
under the product sup metric, provided the side transformer $T_0$ is identical in the two compared circuits and the tensor constructor does not copy the hole.
\item \textbf{Guarded trace.}
Suppose the two pre-trace feedback maps have product-metric discrepancy at most $\delta$ and share the same certified constants $\alpha_Z,\beta_Z$ at that trace node. Then tracing amplifies the discrepancy by at most
$
1+\frac{\beta_Z}{1-\alpha_Z}.
$
Consequently, if the pre-trace subcontext has gain $L(\mathcal D)$, then
$
L(\Tr^Z(\mathcal D))
\le
\left(1+\frac{\beta_Z}{1-\alpha_Z}\right)L(\mathcal D).
$
\end{itemize}
The closed-context modulus $\kappa(\mathcal C)$ is obtained by the same structural certificate: composition multiplies Lipschitz constants, tensoring takes the maximum under the product sup metric, and each closed trace node is admitted only when its certified external modulus is strictly smaller than one.
\end{lemma}

\begin{theorem}[Contextual equivalence for certified admitted guarded contexts]
\label{thm:congruence}
Let $\mathcal C[\cdot]$ be a certified admitted linear guarded one-hole context in the sense of Definition~\ref{def:admitted-linear-context}. Let $\mathsf M$ and $\mathsf N$ be compatible components with the same typed hole interface, and assume that all side components and all certificates are identical in the two compared circuits. Assume also that the resulting closed-loop operators are self-maps of a common invariant ball and are contractions with certified modulus $\kappa(\mathcal C)<1$.

If
$
d_\infty^{X\leftarrow Y}(\mathcal{T}_{\mathsf M},\mathcal{T}_{\mathsf N})\le \varepsilon,
$
then the contextual fixed-point values satisfy
$
\|V_{\mathcal C[\mathsf M]}-V_{\mathcal C[\mathsf N]}\|_\infty
\le
\frac{L(\mathcal C)}{1-\kappa(\mathcal C)}\,\varepsilon.
$
\end{theorem}

Theorem~\ref{thm:congruence} is an error-propagation law for decision circuits: a local typed semantic mismatch $\varepsilon$ is amplified first by the linear context gain $L(\mathcal C)$ and then by the closed-loop sensitivity factor $1/(1-\kappa(\mathcal C))$. The theorem is deliberately typed: it never compares value functions living on different state spaces until an explicit adapter has transported them to a common interface.

\section{Coalgebraic Abstraction and Representation}
\label{sec:abstraction}

Sections~\ref{sec:oddc}--\ref{sec:compositionality} developed a typed, compositional Bellman semantics:
open one-step components induce continuation transformers, and closing feedback yields a fixed-point value.
This section connects that semantics to \emph{representation and abstraction}.
We proceed in three steps.
First, we restate an MDP as a \emph{coalgebra} (a one-step dynamics map into a structured functor),
so that an abstraction becomes a (near-)coalgebra morphism in the sense of universal coalgebra \citep{rutten2000universal,jacobs2017introduction}.
Second, we formalize deterministic state abstraction as a commuting diagram between concrete and abstract kernels;
exact commutation implies value preservation via an operator-intertwining argument,
while approximate commutation yields an explicit simulation-style value bound.
Third, we expose a precise bridge to Section~\ref{sec:compositionality}:
the abstraction error appears as a typed \emph{transformer-level} mismatch, and therefore composes through
arbitrary circuit contexts by Theorem~\ref{thm:congruence}.

\subsection{MDP homomorphisms as commuting diagrams}
\label{sec:hom-diagrams}

Fix a reward signal space $R$ (e.g., $R=\R$) and an action space $A$.
A one-step controlled kernel $K:S\otimes A\to S\otimes R$ may be viewed, pointwise in the action, as a family
\[
c_a:S\longrightarrow \Delta(S\otimes R),
\qquad c_a(s):=K(\cdot\mid s,a).
\]
When $A$ is finite, or when a suitable measurable function-space structure is fixed, this family can be packaged as a coalgebra
$c:S\to(\Delta(S\otimes R))^A$ for the functor $F(X)=(\Delta(X\otimes R))^A$.
To avoid relying on function-space measurability, the formal results below use the equivalent pointwise kernel-commutation conditions in $a$.
Under this lens, an abstraction map $\phi:S\to \widehat S$ should make the one-step dynamics and rewards commute after pushing forward along $\phi$; the remainder of this section makes this concrete and quantitative.

For a measurable map $\phi:S\to\widehat S$ and a probability measure $\mu$ on $S$,
the pushforward $\phi_\#\mu$ is defined by $\phi_\#\mu(B):=\mu(\phi^{-1}(B))$ for measurable $B\subseteq\widehat S$.
We use total variation distance
\[
\TV(\mu,\nu):=\sup_{\|f\|_\infty\le 1}\big|\E_\mu[f]-\E_\nu[f]\big|,
\]
so that for any bounded measurable $f$,
\begin{equation}
\label{eq:tv-test-function}
\big|\E_\mu[f]-\E_\nu[f]\big| \le \|f\|_\infty\,\TV(\mu,\nu).
\end{equation}

\begin{definition}[Reward-preserving MDP homomorphism]
\label{def:mdp-hom}
Let $\mathcal{M}=(S,A,P,r,\gamma)$ and $\widehat{\mathcal{M}}=(\widehat S,A,\widehat P,\widehat r,\gamma)$ share the same action space.
A measurable surjection $\phi:S\to\widehat S$ is a (state) homomorphism if for all $s\in S$ and $a\in A$,
\[
\widehat r(\phi(s),a)=r(s,a),
\qquad
\phi_\# P(\cdot\mid s,a)=\widehat P(\cdot\mid \phi(s),a).
\]
Equivalently, in $\Stoch$ the following kernel-level square commutes:
\[
\begin{tikzcd}[column sep=large,row sep=large]
S\otimes A \arrow[r,"P"] \arrow[d,"\phi\otimes \id_A"'] & S \arrow[d,"\phi"] \\
\widehat S\otimes A \arrow[r,"\widehat P"'] & \widehat S
\end{tikzcd}
\qquad\text{and rewards are constant on fibers of }\phi.
\]
\end{definition}

Given an abstract policy $\widehat\pi:\widehat S\to\Delta(A)$, we define its lifted concrete policy by
$
\pi(a\mid s):=\widehat\pi(a\mid \phi(s)).
$
Conceptually, $\widehat\pi$ acts on a representation $\widehat s=\phi(s)$ and is then pulled back to the concrete state.

Define the pullback operator $\phi^*:\mathcal{B}(\widehat S)\to\mathcal{B}(S)$ by $\phi^*\widehat V := \widehat V\circ \phi$.
Then $\|\phi^*\widehat V\|_\infty \le \|\widehat V\|_\infty$, so comparing values via pullback never amplifies sup-norm errors.

\begin{lemma}[Bellman intertwining under a homomorphism]
\label{lem:intertwine}
Let $\phi$ be a homomorphism and let $\widehat\pi$ and $\pi=\widehat\pi\circ\phi$ be as above.
Let $\mathcal{T}^{\pi}$ and $\widehat{\mathcal{T}}^{\widehat\pi}$ denote the policy-evaluation Bellman operators
on $\mathcal{B}(S)$ and $\mathcal{B}(\widehat S)$ respectively.
Then for all bounded $\widehat V\in\mathcal{B}(\widehat S)$,
$
\mathcal{T}^{\pi}(\phi^*\widehat V) \;=\; \phi^*(\widehat{\mathcal{T}}^{\widehat\pi}\widehat V).
$
\end{lemma}

The identity is an \emph{operator-level} commuting diagram:
pull back an abstract value $\widehat V$ to the concrete space and apply the concrete Bellman backup,
or apply the abstract backup first and then pull back---the results coincide.
This is exactly the categorical/coalgebraic idea that the abstraction commutes with one-step dynamics.

\begin{theorem}[Exact abstraction preserves policy values]
\label{thm:exact-hom}
Assume $\phi$ is a homomorphism.
For any abstract policy $\widehat\pi:\widehat S\to\Delta(A)$ and its lift $\pi=\widehat\pi\circ\phi$,
\[
V^{\pi}(s)=\widehat V^{\widehat\pi}(\phi(s))\quad \forall s\in S,
\qquad\text{i.e.,}\qquad
V^\pi=\phi^*\widehat V^{\widehat\pi}.
\]
\end{theorem}

By Lemma~\ref{lem:intertwine}, the pullback $\phi^*\widehat V$ evolves under $\mathcal{T}^{\pi}$ exactly as
$\widehat V$ evolves under $\widehat{\mathcal{T}}^{\widehat\pi}$.
Since both are $\gamma$-contractions in $\|\cdot\|_\infty$ (Section~\ref{sec:oddc}), they admit unique fixed points,
hence the fixed points correspond by pullback.

To state optimal-value preservation without technical distractions, we record a mild assumption that ensures the optimality
operator is well-defined on bounded measurable functions.

\begin{assumption}[Well-posed optimality backup]
\label{ass:opt-wellposed}
Either (i) $A$ is finite, or (ii) $A$ is a standard Borel space and the MDP satisfies conditions under which the supremum in the
optimality operator admits a measurable selection (so the Bellman optimality operator is well-defined on $\mathcal{B}(S)$).
\end{assumption}

\begin{lemma}[Optimality intertwining]
\label{lem:opt-intertwine}
Under Assumption~\ref{ass:opt-wellposed}, define the optimality operators
\begin{align*}
(\mathcal{T}^\star V)(s)
&:=\sup_{a\in A}\Big(r(s,a)+\gamma\!\int V(s')\,P(ds'\mid s,a)\Big),\\
(\widehat{\mathcal{T}}^\star \widehat V)(\widehat s)
&:=\sup_{a\in A}\Big(\widehat r(\widehat s,a)+\gamma\!\int \widehat V(\widehat s')\,\widehat P(d\widehat s'\mid \widehat s,a)\Big).
\end{align*}
If $\phi$ is a homomorphism, then for all $\widehat V\in\mathcal{B}(\widehat S)$,
$
\mathcal{T}^\star(\phi^*\widehat V)=\phi^*(\widehat{\mathcal{T}}^\star \widehat V).
$
\end{lemma}

\begin{corollary}[Exact abstraction preserves optimal values]
\label{cor:opt-preserve}
Under Assumption~\ref{ass:opt-wellposed} and the homomorphism conditions of Definition~\ref{def:mdp-hom},
\[
V^\star=\phi^*\widehat V^\star,
\qquad\text{i.e.,}\qquad
V^\star(s)=\widehat V^\star(\phi(s))\ \ \forall s\in S.
\]
\end{corollary}

\subsection{Approximate commutation and value distortion}
\label{sec:approx-hom}

Exact commutation is often too strict for learned representations:
$\phi$ may be trained from data, and the abstract kernel $(\widehat P,\widehat r)$ may be fit only approximately. This mirrors the motivation behind approximate MDP homomorphisms, bisimulation metrics, and state-abstraction bounds \citep{givan2003equivalence,ravindran2004approximate,li2006towards,ferns2004metrics,ferns2011bisimulation}.
We therefore quantify abstraction quality by a uniform reward mismatch and a uniform pushforward transition mismatch in total variation.

\begin{theorem}[Approximate abstraction yields an explicit value bound]
\label{thm:approx-hom}
Assume there exist $\widehat r,\widehat P$ such that
\[
\sup_{s\in S,\,a\in A}|r(s,a)-\widehat r(\phi(s),a)|\le \varepsilon_r,
\qquad
\sup_{s\in S,\,a\in A}\TV\!\big(\phi_\#P(\cdot\mid s,a),\widehat P(\cdot\mid \phi(s),a)\big)\le \varepsilon_P.
\]
Assume also a common reward bound $|r(s,a)|\le R_{\max}$ and $|\widehat r(\widehat s,a)|\le R_{\max}$, and set $V_{\max}:=R_{\max}/(1-\gamma)$.
Then for any abstract policy $\widehat\pi$ and its lift $\pi=\widehat\pi\circ\phi$,
\[
\|V^\pi-\phi^*\widehat V^{\widehat\pi}\|_\infty
\;\le\;
\frac{\varepsilon_r+\gamma V_{\max}\varepsilon_P}{1-\gamma}.
\]
\end{theorem}

There are two error sources, and both propagate geometrically through discounting.
The immediate reward mismatch contributes at most $\varepsilon_r$ per step, hence $\varepsilon_r/(1-\gamma)$ overall.
The transition mismatch only affects the \emph{next-value expectation}:
by~\eqref{eq:tv-test-function} with $f=\widehat V^{\widehat\pi}$ and $\|\widehat V^{\widehat\pi}\|_\infty\le V_{\max}$,
the one-step expectation error is at most $\gamma V_{\max}\varepsilon_P$; summing over time yields the second term.

\begin{corollary}[Approximate commutation implies an intertwining defect bound]
\label{cor:op-mismatch}
Under the assumptions of Theorem~\ref{thm:approx-hom}, let $T:=\mathcal{T}^{\pi}$ on $\mathcal{B}(S)$
and $\widehat T:=\widehat{\mathcal{T}}^{\widehat\pi}$ on $\mathcal{B}(\widehat S)$.
Then on the ball $\mathbb{B}_{V_{\max}}(\widehat S):=\{\widehat V:\|\widehat V\|_\infty\le V_{\max}\}$,
\[
\sup_{\widehat V\in \mathbb{B}_{V_{\max}}(\widehat S)}
\big\|\,T(\phi^*\widehat V) - \phi^*(\widehat T\widehat V)\,\big\|_\infty
\ \le\ \varepsilon_r+\gamma V_{\max}\varepsilon_P.
\]
\end{corollary}

Corollary~\ref{cor:op-mismatch} is a \emph{typed transformer mismatch}.  To avoid mixing concrete and abstract interfaces, define two adapter-level transformers with the same type
\[
\mathbb{B}_{V_{\max}}(\widehat S)\longrightarrow \mathbb{B}_{V_{\max}}(S):
A_\phi(T)(\widehat V):=T(\phi^*\widehat V),\qquad
\widehat A_\phi(\widehat T)(\widehat V):=\phi^*(\widehat T\widehat V).
\]
Corollary~\ref{cor:op-mismatch} says precisely that
\[
d_\infty\big(A_\phi(T),\widehat A_\phi(\widehat T)\big)\le
\varepsilon_r+\gamma V_{\max}\varepsilon_P
\]
on this common adapter type.  Therefore abstraction can be inserted into a larger circuit only after these adapters have made the concrete and abstract blocks type-compatible.

\begin{proposition}[Typed abstraction error propagates through circuit contexts]
\label{prop:abstraction-context}
Fix a well-typed one-hole context $\mathcal{C}[\cdot]$ generated from $\circ,\otimes,\Tr$ whose hole has the common adapter type
$\mathbb{B}_{V_{\max}}(\widehat S)\to \mathbb{B}_{V_{\max}}(S)$ and whose induced closed-loop operator is contractive.  Let
$\varepsilon:=\varepsilon_r+\gamma V_{\max}\varepsilon_P$ denote the local adapter-level defect from Corollary~\ref{cor:op-mismatch}.  Then the contextual fixed-point values produced by plugging the adapted concrete block $A_\phi(T)$ and the adapted abstract block $\widehat A_\phi(\widehat T)$ into $\mathcal{C}[\cdot]$ lie in the same external value space and satisfy
\[
\big\|V_{\mathcal{C}[A_\phi(T)]}-V_{\mathcal{C}[\widehat A_\phi(\widehat T)]}\big\|_\infty
\le\frac{L(\mathcal{C})}{1-\kappa(\mathcal{C})}\,\varepsilon,
\]
with $L(\mathcal{C})$ and $\kappa(\mathcal{C})$ as in Definition~\ref{def:context-gain} and Lemma~\ref{lem:context-recursion}.  In particular, the proposition compares only values that have first been transported to a common typed interface; no norm is taken directly between a function on $S$ and a function on $\widehat S$.
\end{proposition}

\subsection{Remark: symmetry as (approximate) commuting diagrams}
\label{sec:symmetry-remark}
\begin{remark}
Approximate symmetries (including learned linear or nonlinear transforms) fit the same near-commutation template.
To incorporate action relabelings, consider a pair $(\phi,\eta)$ with $\phi:S\to S$ and $\eta:A\to A$ measurable.
The symmetry condition is the exact commutation
\[
r(\phi(s),\eta(a))=r(s,a),
\qquad
\phi_\#P(\cdot\mid s,a)=P(\cdot\mid \phi(s),\eta(a)),
\]
and approximate symmetries replace these equalities by uniform sup-norm and TV errors.
This is precisely an (approximate) homomorphism of $\mathcal{M}$ to itself with action relabeling,
so Theorem~\ref{thm:approx-hom} provides an immediate value distortion bound,
which then composes further through any circuit context via Proposition~\ref{prop:abstraction-context}.
\end{remark}

\section{Quantale Contracts and Correct-by-Construction Safety}
\label{sec:contracts}

Sections~\ref{sec:oddc}--\ref{sec:abstraction} developed the additive discounted semantics via Banach (metric) fixed points.
Many safety/resource/risk requirements, however, are most naturally phrased as \emph{order} constraints:
``the induced closed-loop guarantee is below a specification bound''.
This section introduces an order-theoretic contract layer over a quantale-valued value space.
The development has three steps.
(i) We define quantale-valued value functions and (open) \emph{contract transformers} in continuation-passing style,
matching the typed Bellman transformers of Section~\ref{sec:compositionality}.
(ii) We recall least fixed points on complete lattices and show that \emph{inductive} contracts
(pre-fixed points) soundly bound the closed-loop semantics.
(iii) We state explicit \emph{wiring laws} for transformers under $\circ,\otimes,\Tr$ and derive
a correct-by-construction rule: local inductive bounds lift compositionally to global guarantees.

\subsection{Why safety is contract lifting}

Safety/resource/risk constraints are upper bounds in an ordered algebra of guarantees.
Quantales provide a minimal compositional structure \citep{lawvere1973metric,eklund2018fundamentals,yetter1990quantales}: $\otimes_{\cQ}$ composes step-wise guarantees,
and joins $\bigvee$ capture choice/nondeterminism (e.g., demonic worst-case or specification disjunction).
If the closed-loop semantics is given by a monotone operator admitting a least fixed point,
then a local inductive inequality $T(C)\le C$ certifies a global bound $\lfp(T)\le C$, following the same lattice-theoretic logic behind fixed-point semantics and abstract interpretation \citep{tarski1955lattice,kleene1952introduction,studer2008proof,cousot1977abstract}.
This is the precise sense in which safety is \emph{contract lifting}.

\paragraph{Notation.}
We use $\otimes$ for the monoidal product of system interfaces (cartesian product in $\Stoch$),
and $\otimes_{\cQ}$ for the monoidal product in the quantale (applied pointwise on value functions when needed).

\begin{definition}[Quantale-valued value space]
Fix a quantale $(\cQ,\le,\bigvee,\otimes_{\cQ},e)$ and a state space $S$.
Define the value space $\cV(S):=\cQ^S$ with pointwise order.
Let $\bot\in\cV(S)$ denote the pointwise least element.
\end{definition}

\begin{definition}[Closed-loop contract]
A (closed-loop) contract is a bound $C\in\cV(S)$.
A closed loop $(\mathsf{M},\pi)$ \emph{satisfies} $C$ if its semantic value exists and
$
V^\pi \le C \quad \text{pointwise on } S.
$
\end{definition}

For compositional reasoning we work with \emph{open} components and their continuation transformers.
Intuitively, an open transformer takes a desired post-condition on the output state
and returns a sufficient pre-condition on the input state.

\begin{definition}[Contract transformer]
\label{def:contract-transformer}
An open component with input state space $X$ and output state space $Y$
is assigned a monotone map
$
\mathcal{T}:\cV(Y)\to \cV(X),
$
called its \emph{contract transformer}.
Given a post-contract $C_Y\in\cV(Y)$, the induced pre-contract is $C_X:=\mathcal{T}(C_Y)$.
\end{definition}

\paragraph{Interpretation (why this matches Bellman transformers).}
This contract layer is the order-valued analogue of the Banach Bellman layer developed earlier.
For example, when $\cQ=[0,\infty]$ with $\otimes_{\cQ}=+$, a value $V\in\cV(S)$ can represent
a nonnegative accumulated badness, risk, resource usage, or violation budget.
A contract transformer $\mathcal{T}:\cV(Y)\to\cV(X)$ maps a post-bound on the output interface
to a sufficient pre-bound on the input interface.
Thus the inequality
$
\mathcal{T}(C_Y)\le C_X
$
is a one-step proof obligation: assuming the continuation satisfies $C_Y$, the current component
satisfies the stronger or equal bound $C_X$.

Standard signed discounted reward is not treated as a quantale-valued contract in this section;
it is handled by the Banach fixed-point semantics of Sections~\ref{sec:oddc}--\ref{sec:compositionality}.

\subsection{Least fixed points and inductive contracts}
\label{sec:lfp-contract}

A closed loop induces a monotone endomorphism on $\cV(S)$:
$
\mathcal{T}_{\mathsf{M},\pi}:\cV(S)\to \cV(S),
$
and its meaning is defined as a least fixed point whenever that exists:
$
V^\pi := \lfp(\mathcal{T}_{\mathsf{M},\pi}).
$
This order-theoretic view is particularly natural for ``badness'' measures
(e.g., risk, constraint-violation probability, accumulated resource usage),
where smaller is safer.

\begin{definition}[$\omega$-continuity]
A monotone map $T:\cV(S)\to\cV(S)$ is $\omega$-continuous if for every increasing chain
$V_0\le V_1\le\cdots$,
$
T\Big(\bigvee_{n\ge 0} V_n\Big)=\bigvee_{n\ge 0}T(V_n).
$
\end{definition}

\begin{lemma}[Least fixed point via Kleene iteration]
\label{lem:kleene}
If $T:\cV(S)\to\cV(S)$ is monotone and $\omega$-continuous, then $T$ admits a least fixed point
$
\lfp(T)=\bigvee_{n\ge 0} T^n(\bot).
$
\end{lemma}

\begin{lemma}[Inductive contract check: pre-fixed points bound the least fixed point]
\label{lem:prefp-bounds-lfp}
Let $T:\cV(S)\to\cV(S)$ be monotone and admit a least fixed point $\lfp(T)$.
If $C\in\cV(S)$ satisfies
$
T(C)\le C,
$
then $\lfp(T)\le C$.
\end{lemma}

Think of $C$ as a candidate invariant bound.
The inequality $T(C)\le C$ states: assuming the future behavior respects bound $C$,
one step of closed-loop evolution produces a bound no larger than $C$ again.
Lemma~\ref{lem:prefp-bounds-lfp} then formalizes \emph{correct-by-construction} safety:
the least fixed point (the strongest semantics consistent with the dynamics) is automatically below $C$.

To mirror the Banach trace from Section~\ref{sec:oddc} in an order-theoretic setting,
we use a least-fixed-point trace.

\begin{definition}[Least-fixed-point trace]
\label{def:lfp-trace}
Let $\mathbf{CPO}_{\omega,\bot}$ be the category of pointed $\omega$-complete partial orders and $\omega$-continuous maps,
with cartesian product as the monoidal product.
For a morphism $f:X\times Z\to Y\times Z$ in this category, write
$f=(f_Y,f_Z)$ and, for each $x\in X$, define
$
 g_x:Z\to Z,
\qquad
 g_x(z):=f_Z(x,z).
$
Because $f_Z$ is jointly $\omega$-continuous, each $g_x$ is $\omega$-continuous and the parameterized least-fixed-point map
$x\mapsto z_x:=\lfp(g_x)$ is $\omega$-continuous by the standard parameterized Kleene construction.  Define
$
\Tr^{Z}_{X,Y}(f)(x) := f_Y(x,z_x),
\qquad
z_x := \lfp(g_x)=\bigvee_{n\ge0}g_x^n(\bot_Z).
$
This is the least-fixed-point trace used in the contract layer; it is applied only to morphisms for which the pre-trace map is jointly $\omega$-continuous.
\end{definition}

In the discounted additive case of Section~\ref{sec:oddc}, the Bellman operator is a $\gamma$-contraction on a Banach space,
hence has a unique fixed point.
Whenever it is also monotone on the pointwise order, that unique fixed point coincides with the least fixed point
on the corresponding complete lattice (restricted to an invariant order interval),
so Lemma~\ref{lem:prefp-bounds-lfp} applies unchanged.

\subsection{Contract lifting under wiring}
\label{sec:contract-lifting}

The general theorem below is not meant to be vacuous: the following running instance satisfies the required laws for temporal series, separable parallel products, and least-fixed-point feedback.

\begin{proposition}[Concrete additive contract transformer]
\label{prop:concrete-contract-instance}
Let $\cQ=[0,\infty]$ with $\otimes_{\cQ}=+$, and let $K(dy\mid x)$ be a Markov kernel from $X$ to $Y$.
For this concrete probabilistic instance, use the measurable contract space
$
\mathcal V_+(S):=\{C:S\to[0,\infty]\mid C\text{ is measurable}\},
$
ordered pointwise.
For a nonnegative measurable one-step cost $c:X\to[0,\infty]$ and discount $\gamma\in(0,1)$, define
\[
(\mathcal T_{K,c}C)(x):=c(x)+\gamma\int_Y C(y)\,K(dy\mid x),\qquad C\in \mathcal V_+(Y),
\]
where the integral is the extended nonnegative integral.
Then $\mathcal T_{K,c}:\mathcal V_+(Y)\to\mathcal V_+(X)$ is monotone and $\omega$-continuous.  Moreover:
\begin{enumerate}
\item If $K_1:X\to\Delta(Y)$ and $K_2:Y\to\Delta(Z)$ are wired as two temporal micro-steps with costs $c_1$ and $c_2$, then the macro transformer equals $\mathcal T_{K_1,c_1}\circ \mathcal T_{K_2,c_2}$; hence the second cost and continuation are one discount factor deeper.
\item For independent parallel components with product kernel $K_1\otimes K_2$, separable cost $c_1\otimes_{\cQ}c_2$ (pointwise addition), and separable measurable contracts $C_1\otimes_{\cQ}C_2$, the product transformer satisfies
\[
\mathcal T_{K_1\otimes K_2,\,c_1\otimes_{\cQ}c_2}(C_1\otimes_{\cQ}C_2)
=(\mathcal T_{K_1,c_1}C_1)\otimes_{\cQ}(\mathcal T_{K_2,c_2}C_2).
\]
\item For a feedback circuit whose pre-trace order-valued transformer is a morphism in $\mathbf{CPO}_{\omega,\bot}$, closing the feedback wire is interpreted by the least-fixed-point trace of Definition~\ref{def:lfp-trace}.
\end{enumerate}
\end{proposition}

For the remainder of the section we state the contract-lifting theorem for any semantics satisfying the same compilation laws.  This separates the concrete probabilistic instance above from the purely order-theoretic proof.

\begin{assumption}[Compositional transformer laws]
\label{ass:compilation-laws}
There is a semantics assignment $\llbracket\cdot\rrbracket$ mapping any well-typed open circuit
to a monotone, $\omega$-continuous contract transformer between the relevant pointed $\omega$-cpos.
Moreover, whenever a feedback constructor is used, the corresponding pre-trace transformer is a morphism in $\mathbf{CPO}_{\omega,\bot}$; equivalently, its components are jointly $\omega$-continuous in the post-contract and feedback-contract arguments.
Finally, $\llbracket\cdot\rrbracket$ respects wiring as follows (up to canonical reindexing of interfaces):
\begin{align*}
\llbracket \mathsf{M}_2\circ \mathsf{M}_1 \rrbracket
&= \llbracket \mathsf{M}_1 \rrbracket \circ \llbracket \mathsf{M}_2 \rrbracket,\\
\llbracket \mathsf{M}_1\otimes \mathsf{M}_2 \rrbracket
&= \llbracket \mathsf{M}_1 \rrbracket \otimes \llbracket \mathsf{M}_2 \rrbracket,\\
\llbracket \Tr(\mathsf{F}) \rrbracket
&= \Tr(\llbracket \mathsf{F}\rrbracket),
\end{align*}
where the rightmost $\Tr$ is the least-fixed-point trace from Definition~\ref{def:lfp-trace}.
\end{assumption}

It is the order-theoretic analogue of Section~\ref{sec:compositionality}:
typed open semantics is continuation passing, series wiring is function composition,
parallel wiring is product/tensor, and closing feedback is trace.
Once these laws hold, contract lifting follows by pure order reasoning.

\begin{theorem}[Contract lifting under composition and feedback]
\label{thm:contract-lifting}
Assume Assumption~\ref{ass:compilation-laws}.
Then inductive contracts are preserved by wiring in the following sense.

\begin{itemize}
\item \textbf{(Series / assume--guarantee).}
Let $\mathsf{M}_1:X\to Y$ and $\mathsf{M}_2:Y\to Z$ be open circuits with transformers
$\mathcal{T}_1:=\llbracket \mathsf{M}_1\rrbracket:\cV(Y)\to\cV(X)$ and
$\mathcal{T}_2:=\llbracket \mathsf{M}_2\rrbracket:\cV(Z)\to\cV(Y)$.
Suppose there exist bounds $C_Z\in\cV(Z)$, $C_Y\in\cV(Y)$, and $C_X\in\cV(X)$ such that
$
\mathcal{T}_2(C_Z)\le C_Y,
\qquad
\mathcal{T}_1(C_Y)\le C_X.
$
Then the composed circuit satisfies the induced bound
$
\llbracket \mathsf{M}_2\circ \mathsf{M}_1 \rrbracket(C_Z)\;=\;(\mathcal{T}_1\circ\mathcal{T}_2)(C_Z)\ \le\ C_X.
$
In particular, if $X=Z=S$ and $C_X=C_Z=:C_S$, this yields an inductive closed-loop contract
$(\mathcal{T}_1\circ\mathcal{T}_2)(C_S)\le C_S$.

\item \textbf{(Parallel / separable product contracts).}
Let $\mathsf{M}_i:X_i\to Y_i$ with transformers $\mathcal{T}_i:\cV(Y_i)\to\cV(X_i)$.
For two contracts $C_1\in\cV(Y_1)$ and $C_2\in\cV(Y_2)$, write their separable product contract as
\[
(C_1\otimes_{\cQ}C_2)(y_1,y_2):=C_1(y_1)\otimes_{\cQ} C_2(y_2).
\]
If $\mathcal{T}_i(C_{Y_i})\le C_{X_i}$ for $i=1,2$, then for the parallel circuit
\[
(\mathcal{T}_1\otimes \mathcal{T}_2)(C_{Y_1}\otimes_{\cQ} C_{Y_2})
\ \le\
C_{X_1}\otimes_{\cQ} C_{X_2},
\]
pointwise on $X_1\otimes X_2$. This statement is for separable product contracts; arbitrary contracts on
$Y_1\otimes Y_2$ need not decompose in this form.

\item \textbf{(Feedback / trace).}
Let $\mathsf{F}$ be an open circuit with a feedback wire of type $Z$, and write its transformer in components as
\begin{align*}
F:=\llbracket \mathsf{F}\rrbracket:
\cV(Y)\times\cV(Z)\longrightarrow \cV(X)\times\cV(Z),
F(C_Y,C_Z)=(F_X(C_Y,C_Z),F_Z(C_Y,C_Z)).
\end{align*}
Fix a post-bound $C_Y\in\cV(Y)$. If there exist bounds $C_X\in\cV(X)$ and $C_Z\in\cV(Z)$ such that
$
F_Z(C_Y,C_Z)\le C_Z,
\qquad
F_X(C_Y,C_Z)\le C_X,
$
then the traced transformer $\Tr^Z(F)$ is well-defined by the least-fixed-point trace and satisfies
$
\Tr^Z(F)(C_Y)\le C_X.
$
\end{itemize}
\end{theorem}

The series and parallel cases are immediate from monotonicity and the transformer wiring equalities in
Assumption~\ref{ass:compilation-laws}.
For feedback, Definition~\ref{def:lfp-trace} reduces closing a loop to a least fixed point of the $Z$-component;
the inductive inequality ensures this fixed point is bounded by $C_Z$ via Lemma~\ref{lem:prefp-bounds-lfp},
and the $Y$-projection then yields the traced bound.
Complete proofs are deferred to Appendix~\ref{app:contracts-proofs}.

\section{Extensions}
\label{sec:extensions}

The previous sections establish a typed, compositional semantics for discounted decision components:
open circuits denote continuation-style transformers (Sections~\ref{sec:oddc}--\ref{sec:compositionality}),
feedback closure is a fixed point (Section~\ref{sec:oddc} and Section~\ref{sec:contracts}),
and representation maps are commuting diagrams with quantitative mismatch bounds (Section~\ref{sec:abstraction}).
This section shows that the same semantic viewpoint extends to three ubiquitous ``real-world'' wrinkles in a way
that remains compatible with contextual congruence (Theorem~\ref{thm:congruence}):
(i) \emph{partial observability} becomes a change of the state object from $S$ to the belief object $\Delta(S)$,
so the same fixed-point semantics applies after lifting;
(ii) \emph{off-policy evaluation} becomes a change-of-measure on trajectory space, and under genuine circuit decompositions
the Radon--Nikodym derivative admits a module-wise factorization (enabling modular control of importance-weight dispersion);
(iii) \emph{nonstationarity} becomes a time-indexed path of semantic operators, and fixed-point stability yields
tracking bounds that compose with our operator-mismatch metrics.

\subsection{Partial observability as belief-state lifting}
\label{sec:belief-mdp}

A discounted POMDP is $(S,A,O,P,G,r,\gamma,\nu_0)$, following the standard partially observable control formulation \citep{smallwood1973optimal,kaelbling1998planning}. Here $S$ is the state space, $A$ is the action space, $O$ is the observation space,
$
P(\cdot\mid s,a)
$
is the state-transition kernel on $S$,
$
G(\cdot\mid s',a)
$
is the observation kernel on $O$, $r:S\times A\to\R$ is bounded measurable, $\gamma\in(0,1)$, and $\nu_0\in\Delta(S)$ is the initial belief.

\begin{assumption}[Belief-lifting regularity]
\label{ass:belief-regularity}
The spaces $S,A,O$ are standard Borel. The transition and observation kernels are measurable Markov kernels, and regular conditional distributions are chosen so that there exists a measurable Bayes update
$
\tau:\Delta(S)\times A\times O\to\Delta(S).
$
For every prior belief $b\in\Delta(S)$ and action $a\in A$, $\tau(b,a,o)$ is a version of the posterior law of $s'$ given the observation $o$ under the joint kernel
$
b(ds)\,P(ds'\mid s,a)\,G(do\mid s',a).
$
On observations with zero predictive probability, $\tau$ is defined arbitrarily; this does not change any induced trajectory law.
\end{assumption}

Let $B:=\Delta(S)$ be the belief space equipped with its standard measurable structure. Given $(b,a)\in B\times A$, the predictive observation law is
\[
\ell(E\mid b,a)
:=
\int_S\int_S G(E\mid s',a)\,P(ds'\mid s,a)\,b(ds),
\qquad E\subseteq O\ \text{measurable}.
\]
The posterior belief is
$
b'=\tau(b,a,o).
$
This induces a belief transition kernel $P_B(\cdot\mid b,a)$ on $B$ by pushing $\ell(\cdot\mid b,a)$ through $\tau$:
\[
P_B(C\mid b,a)
:=
\int_O \mathbf 1\{\tau(b,a,o)\in C\}\,\ell(do\mid b,a),
\qquad C\subseteq B\ \text{measurable}.
\]
The one-step reward is lifted linearly:
$
r_B(b,a):=\int_S r(s,a)\,b(ds),
$
which remains bounded with $\|r_B\|_\infty\le \|r\|_\infty$.

We use the convention that the initial belief $b_0=\nu_0$ is specified before any initial observation is incorporated. After action $a_t$ and observation $o_{t+1}$, the history is
$
h_{t+1}=(a_0,o_1,\ldots,a_t,o_{t+1}),
$
and the belief is
$
b_{t+1}=\mathbb P(s_{t+1}\in\cdot\mid h_{t+1}).
$
Equivalently, $b_t$ is the posterior belief available when choosing $a_t$. The induced belief policy may be randomized and time-dependent; the statement is equality in law, not pathwise equality for every realized history.

\begin{lemma}[History policies induce nonstationary belief policies]
\label{lem:belief-policy}
Under Assumption~\ref{ass:belief-regularity}, fix an initial belief $\nu_0$. For any randomized history-dependent POMDP policy $\pi$, there exists a possibly time-dependent Markov kernel
$
\pi_{B,t}:B\to\Delta(A)
$
such that, for every $t$,
$
(b_t,a_t)\ \stackrel{d}{=}\ (b_t^{B},a_t^{B}),
$
where the right-hand side is generated by the belief process with transition $P_B$ and policy sequence $(\pi_{B,t})_{t\ge0}$. Consequently, the expected one-step rewards agree at every time $t$, and the discounted expected returns coincide.
\end{lemma}

\begin{theorem}[Belief-state lifting yields an equivalent fully observed process]
\label{thm:belief-mdp}
Under Assumption~\ref{ass:belief-regularity}, define the belief MDP
$
\mathcal M_B=(B,A,P_B,r_B,\gamma)
$
as above. Then for any randomized history-dependent policy $\pi$ in the original POMDP, there exists a possibly time-dependent randomized belief policy $(\pi_{B,t})_{t\ge0}$ such that the discounted return from the initial belief $b_0=\nu_0$ agrees:
$
V^\pi_{\mathrm{POMDP}}(\nu_0)
=
V^{(\pi_{B,t})}_{\mathcal M_B}(\nu_0).
$
Conversely, any such belief policy is implementable by a history-dependent POMDP policy because $b_t$ is a measurable function of the observation-action history. Hence optimizing over history-dependent POMDP policies is equivalent to optimizing over randomized nonstationary belief policies:
$
V^\star_{\mathrm{POMDP}}(\nu_0)
=
V^\star_{\mathcal M_B,\mathrm{nonstat}}(\nu_0).
$
For stationary belief policies, the fixed-point semantics of Sections~\ref{sec:oddc}--\ref{sec:compositionality} applies directly with state object $B$.
\end{theorem}

Belief lifting changes the state object from $S$ to $B=\Delta(S)$. Once lifted, the belief dynamics is Markov, so our fixed-point semantics and the wiring/compositional results apply without modification, with $S$ replaced by $B$ throughout. Proof details are deferred to Appendix~\ref{app:extensions-proofs}.

\subsection{Off-policy evaluation as compositional change-of-measure}
\label{sec:ope}

Let $\mu(\cdot\mid s)$ be a behavior policy and $\pi(\cdot\mid s)$ a target policy on the same MDP. Off-policy evaluation commonly controls the resulting trajectory-law change via likelihood ratios and related importance-sampling or doubly robust estimators \citep{precup2000eligibility,jiang2016doubly,thomas2016data}.
Let $\mathbb{P}_\mu$ and $\mathbb{P}_\pi$ denote the induced trajectory measures on
$\tau=(s_0,a_0,r_0,s_1,a_1,r_1,\ldots)$, with the same initial state distribution.

\begin{lemma}[Finite-prefix trajectory likelihood ratio in an MDP]
\label{lem:traj-rn}
Fix a horizon $T<\infty$ and let $\mathbb{P}^{T}_\mu$ and $\mathbb{P}^{T}_\pi$ denote the induced laws of the trajectory prefix $(s_0,a_0,r_0,\ldots,s_T)$ with the same initial state distribution.
If $\pi(\cdot\mid s)\ll \mu(\cdot\mid s)$ for all $s$, then $\mathbb{P}^{T}_\pi\ll\mathbb{P}^{T}_\mu$ and
\[
\frac{d\mathbb{P}^{T}_\pi}{d\mathbb{P}^{T}_\mu}
=
\prod_{t=0}^{T-1}
\frac{d\pi(\cdot\mid s_t)}{d\mu(\cdot\mid s_t)}(a_t).
\]
For infinite-horizon trajectory laws, the same products define the finite-prefix likelihood-ratio martingale; the infinite-horizon derivative is its almost-sure limit whenever the limit exists in the usual Radon--Nikodym sense.
\end{lemma}

Lemma~\ref{lem:traj-rn} always factorizes over \emph{time}.
Factorization over \emph{modules} requires additional structural assumptions:
the closed-loop trajectory measure must decompose according to the circuit wiring, and the policies must decompose
in the matching way. We make these assumptions explicit below.

\begin{proposition}[Compositional change-of-measure for decomposable decision circuits]
\label{prop:ope-factor}
Fix a finite horizon $T<\infty$ and consider a closed decision circuit built from series/parallel wiring of subcircuits (Section~\ref{sec:compositionality}).
Assume the behavior and target closed-loop prefix measures admit a \emph{genuine} decomposition matching the wiring:
\begin{itemize}
\item \textbf{Parallel case.} The global state decomposes as $S=S_1\times S_2$, the transition and reward streams
factorize under $\mu$ and $\pi$ into independent subsystems, and the policies factorize as product policies
$\mu=\mu_1\otimes \mu_2$ and $\pi=\pi_1\otimes \pi_2$.
\item \textbf{Series case.} There exists an interface random variable (or interface state) $U_t$ such that, conditional on $U_t$,
the upstream/downstream subtrajectories satisfy the corresponding conditional independence implied by the series wiring,
and the policies factorize conditionally along $U_t$ (chain-rule form).
\end{itemize}
Then the global finite-prefix Radon--Nikodym derivative equals the product of the corresponding local derivatives:
\[
\frac{d\mathbb{P}^{T}_\pi}{d\mathbb{P}^{T}_\mu}
=
\prod_{j\in \text{modules}}
\frac{d\mathbb{P}^{T,(j)}_\pi}{d\mathbb{P}^{T,(j)}_\mu},
\]
where in the series case the product is interpreted via the chain rule for conditional RN derivatives along the interface.
\end{proposition}

This proposition is a conditional factorization statement: the contribution is not that every circuit factorizes,
but that once the closed-loop trajectory law genuinely respects the circuit interface, the Radon--Nikodym
chain rule preserves that modular structure.

Module-wise factorization turns global importance-weight dispersion into a sum of local dispersions in log-space
\emph{when the module weights are independent} (e.g., parallel independent subsystems).
Without independence one still obtains upper bounds via standard inequalities, but the clean additivity becomes an inequality.

\begin{corollary}[Second-moment additivity under independent module weights]
\label{cor:ope-second-moment}
Under the parallel-case assumptions in Proposition~\ref{prop:ope-factor},
if the finite-prefix module-wise importance weights are independent under $\mathbb{P}^{T}_\mu$, then
\[
\log \E_\mu\!\Big[\Big(\tfrac{d\mathbb{P}^{T}_\pi}{d\mathbb{P}^{T}_\mu}\Big)^{\!2}\Big]
=
\sum_{j}\log \E_\mu\!\Big[\Big(\tfrac{d\mathbb{P}^{T,(j)}_\pi}{d\mathbb{P}^{T,(j)}_\mu}\Big)^{\!2}\Big].
\]
\end{corollary}

Lemma~\ref{lem:traj-rn} is the universal MDP-level change-of-measure.
Proposition~\ref{prop:ope-factor} identifies additional structural conditions under which the change-of-measure
respects circuit structure, enabling modular reasoning about OPE variance and failure modes.
Proof sketches are provided in Appendix~\ref{app:extensions-proofs}.

\subsection{Nonstationarity as a time-indexed operator family}
\label{sec:nonstationary}

Let $R_{\max}$ bound one-step rewards uniformly over time, and set $V_{\max}=R_{\max}/(1-\gamma)$.
Let $\mathbb{B}:=\{V:\|V\|_\infty\le V_{\max}\}$.
For each time $t$, let $\mathcal{T}_t:\mathbb{B}\to\mathbb{B}$ be a $\gamma$-contraction.

Define the drift (operator mismatch) by
$
\eta_t := \sup_{V\in\mathbb{B}}\|\mathcal{T}_{t+1}V-\mathcal{T}_tV\|_\infty.
$
This is exactly the $d_\infty$-metric from Section~\ref{sec:contextual} applied to successive operators.

\begin{theorem}[Fixed-point tracking under time-varying contractions]
\label{thm:tracking-functor}
Let $\{\mathcal{T}_t\}_{t\ge 0}$ be $\gamma$-contractions on $\mathbb{B}$ with fixed points $V_t^\star$.
Then for every $t$,
$
\|V_{t+1}^\star-V_t^\star\|_\infty\le \frac{\eta_t}{1-\gamma}.
$
Moreover, the cumulative drift satisfies
$
\|V_{T}^\star-V_0^\star\|_\infty\le \sum_{t<T}\frac{\eta_t}{1-\gamma}.
$
\end{theorem}

Theorem~\ref{thm:tracking-functor} is a direct instance of Lemma~\ref{lem:fp-stability} with $\kappa=\gamma$ and
$d_\infty(\mathcal{T}_{t+1},\mathcal{T}_t)=\eta_t$.
Thus, any compositional operator-mismatch bound (including contextual congruence, Theorem~\ref{thm:congruence})
can be plugged in to control $\eta_t$.

Often one does not compute $V_t^\star$ exactly; instead, one applies a small number of backups per time step.
The next corollary quantifies the tracking error of the simplest update $V_{t+1}:=\mathcal{T}_t V_t$.

\begin{corollary}[Iterate tracking under drifting contractions]
\label{cor:iterate-tracking}
Let $V_{t+1}:=\mathcal{T}_t V_t$ and let $V_t^\star$ be the fixed point of $\mathcal{T}_t$.
Then
\[
\|V_{t+1}-V_{t+1}^\star\|_\infty
\le
\gamma\,\|V_t-V_t^\star\|_\infty + \frac{\eta_t}{1-\gamma}.
\]
Iterating yields
\[
\|V_T-V_T^\star\|_\infty
\le
\gamma^T \|V_0-V_0^\star\|_\infty
+
\sum_{t<T}\gamma^{T-1-t}\frac{\eta_t}{1-\gamma}.
\]
\end{corollary}

Nonstationarity is captured by a time-indexed family of semantic operators.
Theorem~\ref{thm:tracking-functor} controls how the \emph{ground-truth} fixed point drifts,
while Corollary~\ref{cor:iterate-tracking} controls the additional error incurred by limited computation.
Both statements are expressed in the same operator metric used by contextual congruence,
so they compose seamlessly with our circuit-level mismatch bounds.

\section{Scope and Limitations}
\label{sec:limitations}

The results in this paper are intentionally limited to settings where the semantic fixed points are well-defined. In the Banach layer we assume bounded rewards, standard-Borel measurable spaces, and a discount factor $\gamma\in(0,1)$. These assumptions make the Bellman backup a contraction on bounded value functions. Undiscounted, average-reward, unbounded-reward, or risk-sensitive objectives may require different value spaces or different fixed-point principles.

The guarded trace used here is partial. A feedback loop is admitted only when the feedback coordinate is uniformly contractive, and contextual trace laws are used only under this admissibility condition. Thus the paper does not claim that all stochastic kernels, all open decision components, or all Bellman transformers form a total traced symmetric monoidal category.

The paper is also not an algorithmic or empirical contribution. It does not propose a new RL optimizer, estimator, exploration method, or benchmark result. Its purpose is to give a typed compositional semantics for existing discounted Bellman evaluation and to derive error-propagation, abstraction, and contract-lifting theorems from that semantics.

Finally, the extensions in Section~\ref{sec:extensions} are interface checks rather than full replacements for specialized theories. Belief-state lifting recovers the standard fully observed belief process under the usual measurability assumptions. The off-policy factorization result requires genuine decomposition of the trajectory law and policy along the circuit interfaces; without those assumptions, importance weights need not factorize module-wise.

\section{A Minimal Example and Modular Robustness}
\label{sec:example}

This section instantiates the abstract semantics on a minimal yet fully explicit testbed whose purpose is twofold.
First, we give a sanity check that our \emph{parallel} wiring semantics recovers the classical additive value decomposition
for independent subsystems with additive rewards and product policies (Section~\ref{sec:example-parallel}).
Second, we provide a concrete \emph{two-module series} circuit in which we can track how \emph{local} model perturbations
propagate to the \emph{global} fixed-point value through the exact wiring calculus:
\begin{align*}
&\text{(TV/reward mismatch)}
\Rightarrow \text{(single-module operator mismatch)}\\
&\Rightarrow \text{(macro mismatch under series wiring)}
\Rightarrow \text{(fixed-point deviation)}.
\end{align*}
The resulting bound makes a depth-discount phenomenon explicit: perturbations deeper in the feedback loop
are geometrically attenuated by additional factors of $\gamma$ (Section~\ref{sec:example-series}).

\subsection{Parallel product MDP (sanity check)}
\label{sec:example-parallel}

Let $\mathcal{M}_i=(S_i,A_i,P_i,r_i,\gamma)$ be two discounted MDPs on measurable spaces
with bounded rewards $|r_i|\le R_{\max}^{(i)}$.
Define the product MDP on $S=S_1\times S_2$ and $A=A_1\times A_2$ by the product transition kernel: for measurable rectangles $C_1\times C_2$,
\[
P(C_1\times C_2\mid (s_1,s_2),(a_1,a_2))
=
P_1(C_1\mid s_1,a_1)\,P_2(C_2\mid s_2,a_2),
\]
extended uniquely to the product $\sigma$-algebra, and by the additive reward
\[
r((s_1,s_2),(a_1,a_2)) := r_1(s_1,a_1)+r_2(s_2,a_2).
\]
For a product policy $\pi=\pi_1\otimes\pi_2$ (independent actions across subsystems),
the closed-loop trajectory law factorizes across subsystems, and expectation linearity yields the value decomposition.

\begin{proposition}[Additive value factorization for independent parallel products]
\label{prop:parallel-factor}
Under the product dynamics and product policy above,
\[
V^\pi(s_1,s_2)=V^{\pi_1}(s_1)+V^{\pi_2}(s_2)\qquad \forall (s_1,s_2)\in S_1\times S_2.
\]
Equivalently, the fixed point of the parallel Bellman operator factorizes additively.
\end{proposition}

The decomposition is structural: it requires both (i) product dynamics $P=P_1\otimes P_2$ and
(ii) product policy $\pi=\pi_1\otimes\pi_2$.
If either coupling is present (shared noise, coupled rewards, or correlated actions),
the continuation transformer no longer decomposes as a sum of two independent subtransformers,
and cross-terms generally appear.

\subsection{Two-module series robustness (depth-discount)}
\label{sec:example-series}

Consider two open kernels wired in series through an interface $U$ (cf.\ Section~\ref{sec:comp-bellman}):
\[
K_1:S\otimes A_1\to U\otimes R_1,
\qquad
K_2:U\otimes A_2\to S\otimes R_2,
\]
with scalarizations $\rho_i:R_i\to\R$ and a common discount $\gamma\in(0,1)$.
Fix (possibly distinct) policies $\pi_S:S\to\Delta(A_1)$ and $\pi_U:U\to\Delta(A_2)$, and let
$T_1:\mathcal{B}(U)\to\mathcal{B}(S)$ and $T_2:\mathcal{B}(S)\to\mathcal{B}(U)$
denote the induced micro-step Bellman transformers (Section~\ref{sec:typed-bellman}).
We write $\widetilde T_i$ for perturbed versions of these transformers under a perturbed model.

Assume the scalarized rewards are uniformly bounded,
$
|\rho_i(r_i)|\le R_{\max}\qquad \text{for all } r_i\in R_i,\quad i=1,2,
$
so $V_{\max}:=R_{\max}/(1-\gamma)$ bounds all policy-evaluation values in sup norm.
We will work on the closed ball
$\mathbb{B}_S:=\{V\in\mathcal{B}(S):\|V\|_\infty\le V_{\max}\}$ and similarly $\mathbb{B}_U$.

For each module $i\in\{1,2\}$, assume a uniform bound on the scalar one-step reward mismatch
and the transition-kernel mismatch:
\[
\sup_{x,a}\big|\tilde r_i(x,a)- r_i(x,a)\big|\le \varepsilon_r^{(i)},
\qquad
\sup_{x,a}\TV\!\big(\widetilde P^{(i)}(\cdot\mid x,a),P^{(i)}(\cdot\mid x,a)\big)\le \varepsilon_P^{(i)},
\]
where $x$ denotes the appropriate state ($x\in S$ for $i=1$, $x\in U$ for $i=2$).
(When rewards are emitted as random signals $R_i$, $\tilde r_i(x,a)$ and $r_i(x,a)$ are understood as the
scalarized conditional expectations; this convention eliminates ambiguity in the mismatch statement.)

\begin{lemma}[TV controls Bellman expectation mismatch]
\label{lem:traj-tv-bound}
Let $T$ and $\widetilde T$ be Bellman transformers of the same type, with discount $\gamma$ and bounded values.
If the one-step scalar reward mismatch and transition TV mismatch are bounded by $(\varepsilon_r,\varepsilon_P)$ as above,
then on the ball $\mathbb{B}:=\{V:\|V\|_\infty\le V_{\max}\}$,
\[
d_\infty(T,\widetilde T)
:=\sup_{V\in\mathbb{B}}\|TV-\widetilde TV\|_\infty
\;\le\;
\varepsilon_r + \gamma V_{\max}\varepsilon_P
\;=:\;
\epsilon.
\]
\end{lemma}

The two-micro-step series circuit induces a macro Bellman operator on $\mathcal{B}(S)$:
\[
T_{\mathrm{ser}} := T_1\circ T_2:\mathbb{B}_S\to\mathbb{B}_S,
\qquad
\widetilde T_{\mathrm{ser}} := \widetilde T_1\circ \widetilde T_2:\mathbb{B}_S\to\mathbb{B}_S.
\]
Since each micro-step transformer is $\gamma$-Lipschitz in the continuation argument,
$T_{\mathrm{ser}}$ and $\widetilde T_{\mathrm{ser}}$ are $\gamma^2$-contractions.

\begin{lemma}[Series mismatch propagation (depth attenuation)]
\label{lem:series-mismatch}
Let $\epsilon_i:=d_\infty(T_i,\widetilde T_i)$ be the local operator mismatch of module $i$,
where the sup is taken over the corresponding balls $\mathbb{B}_U$ or $\mathbb{B}_S$.
Then
$
d_\infty(T_{\mathrm{ser}},\widetilde T_{\mathrm{ser}})
\;\le\;
\epsilon_1 + \gamma\,\epsilon_2.
$
\end{lemma}

Let $V^\pi$ and $\widetilde V^\pi$ be the unique fixed points of $T_{\mathrm{ser}}$ and $\widetilde T_{\mathrm{ser}}$,
respectively. These fixed points coincide with the infinite-horizon micro-step value function, but are expressed here
through the macro (two-step) operator to expose depth effects algebraically.

\begin{lemma}[Two-module robustness with depth discount]
\label{lem:two-module-robust}
Let $\epsilon_i$ denote the local mismatch bounds from Lemma~\ref{lem:traj-tv-bound}:
$
\epsilon_i \;:=\; \varepsilon_r^{(i)}+\gamma V_{\max}\varepsilon_P^{(i)}.
$
Then the fixed-point values of the original and perturbed two-module series circuit satisfy
$
\|\widetilde V^\pi - V^\pi\|_\infty
\;\le\;
\frac{\epsilon_1 + \gamma\,\epsilon_2}{1-\gamma^2}.
$
\end{lemma}

Lemma~\ref{lem:traj-tv-bound} turns local model perturbations into local operator mismatch.
Lemma~\ref{lem:series-mismatch} shows that, under series wiring, the mismatch of the deeper module (here module 2)
is attenuated by an additional factor $\gamma$ because its effect must pass through one extra discounted continuation.
Finally, Lemma~\ref{lem:two-module-robust} applies contraction stability (Lemma~\ref{lem:fp-stability} with $\kappa=\gamma^2$)
to convert macro operator mismatch into a fixed-point deviation bound.
Full proofs are provided in Appendix~\ref{app:example-proofs}.

{
\small
\nocite{*}
\bibliographystyle{unsrtnat}
\bibliography{ref}
}


\appendix

\section{Proofs}
\label{app:proofs}

This appendix proves the claims in the order in which they appear in the main text. All spaces are understood to satisfy the standing standard-Borel/measurability assumptions from Section~\ref{sec:prelim}. We use the sup norm on bounded real-valued value spaces and the pointwise order on quantale-valued value spaces.

\subsection{Proof of the guarded Banach trace facts}
\label{app:banach-trace-axioms}

\begin{lemma}[Uniform fixed point for a guarded parameterized map]
\label{lem:app-param-banach}
Let $(Z,d_Z)$ be complete, and let $g:X\times Z\to Z$ satisfy
\[
d_Z(g(x,z),g(x,z'))\le c d_Z(z,z')
\qquad \forall x\in X,\ z,z'\in Z
\]
for a constant $c\in[0,1)$. Then for every $x\in X$ there is a unique $z_x\in Z$ such that $z_x=g(x,z_x)$.
\end{lemma}

\begin{proof}
Fix $x\in X$ and define $g_x:Z\to Z$ by $g_x(z)=g(x,z)$. The displayed inequality says exactly that $g_x$ is a contraction with modulus at most $c$. Since $Z$ is complete, Banach's fixed-point theorem gives existence of a fixed point $z_x$. If $z$ and $z'$ are both fixed points, then
\[
d_Z(z,z')=d_Z(g_x(z),g_x(z'))\le c d_Z(z,z').
\]
Subtracting the right-hand side gives $(1-c)d_Z(z,z')\le 0$. Because $1-c>0$ and distances are nonnegative, $d_Z(z,z')=0$, hence $z=z'$. This proves uniqueness.
\end{proof}

\begin{proposition}[Guarded fixed-point identities used in the paper]
On the class of maps for which the displayed feedback operation is guarded, the operation in
Definition~\ref{def:banach-trace} is natural under pre- and post-composition, compatible with products by
side maps, satisfies vanishing for the monoidal unit, and satisfies yanking for the symmetry map.  We do not
claim a global traced-monoidal structure on all partial maps; any additional trace identity must be used only
when both sides of the identity are guarded.
\end{proposition}

\begin{proof}
Write a guarded map $f:X\times Z\to Y\times Z$ as $f=(f_Y,f_Z)$. For each $x$, the trace uses the unique solution $z_x=f_Z(x,z_x)$ and returns $f_Y(x,z_x)$.

Naturality follows by direct substitution. Let $u:X'\to X$ and $v:Y\to Y'$ be maps, and define
\[
f'=(v\times\id_Z)\circ f\circ (u\times\id_Z).
\]
The feedback component of $f'$ at $x'$ is $z\mapsto f_Z(u(x'),z)$. Its unique fixed point is therefore $z_{u(x')}$, the fixed point used by $f$ at $u(x')$. Hence
\[
\Tr^Z(f')(x')
= v\big(f_Y(u(x'),z_{u(x')})\big)
= (v\circ \Tr^Z(f)\circ u)(x').
\]

Compatibility with products is also pointwise. If $f:X\times Z\to Y\times Z$ and $h:X'\to Y'$ does not involve the feedback variable, then the feedback equation for $f\times h$ is the same equation $z=f_Z(x,z)$; the traced output is $(f_Y(x,z_x),h(x'))$, which is $(\Tr^Z f\times h)(x,x')$.

Vanishing for the unit follows because the monoidal unit is a singleton. If $Z=I=\{\star\}$, the only possible feedback point is $\star$, so tracing simply removes the unit coordinate:
\[
\Tr^I(f)(x)=\pi_Y f(x,\star).
\]

Yanking is the special case of the symmetry map $\sigma_{Z,Z}:Z\times Z\to Z\times Z$, $\sigma(z_1,z_2)=(z_2,z_1)$. Tracing the second wire gives the equation $z_2=z_1$, whose unique solution is $z_2=z_1$. The output is therefore $z_1$, so $\Tr^Z(\sigma)=\id_Z$.

These are exactly the guarded fixed-point identities used in the main text.
\end{proof}

\subsection{Proof of Lemma~\ref{lem:value-bounded}}
\begin{proof}
Fix a stationary policy $\pi$ and an initial state $s_0=s$. Let
\[
G_n:=\sum_{t=0}^{n}\gamma^t r(s_t,a_t)
\]
be the $n$-step discounted return. Since $|r(s,a)|\le R_{\max}$ and $\gamma\in(0,1)$,
\[
|G_n|
\le \sum_{t=0}^{n}\gamma^t |r(s_t,a_t)|
\le R_{\max}\sum_{t=0}^{n}\gamma^t
=R_{\max}\frac{1-\gamma^{n+1}}{1-\gamma}
\le \frac{R_{\max}}{1-\gamma}.
\]
Also, for $m>n$,
\[
|G_m-G_n|
\le R_{\max}\sum_{t=n+1}^{m}\gamma^t
\le R_{\max}\frac{\gamma^{n+1}}{1-\gamma}.
\]
The right-hand side tends to zero, so $(G_n)$ is Cauchy in $L^\infty$ and converges almost surely and in $L^\infty$ to
\[
G:=\sum_{t\ge0}\gamma^t r(s_t,a_t).
\]
The same uniform bound passes to the limit: $|G|\le R_{\max}/(1-\gamma)$ almost surely. Therefore
\[
|V^\pi(s)|
=\left|\E^\pi[G\mid s_0=s]\right|
\le \E^\pi[|G|\mid s_0=s]
\le \frac{R_{\max}}{1-\gamma}.
\]
Taking the supremum over $s$ gives $\|V^\pi\|_\infty\le R_{\max}/(1-\gamma)=V_{\max}$.
\end{proof}

\subsection{Proof of Lemma~\ref{lem:BS-banach}}
\begin{proof}
Let $(V_n)_{n\ge1}$ be a Cauchy sequence in $(\mathcal{B}(S),\|\cdot\|_\infty)$. For each fixed $s\in S$,
\[
|V_n(s)-V_m(s)|\le \|V_n-V_m\|_\infty,
\]
so $(V_n(s))_{n\ge1}$ is Cauchy in $\R$. Define $V(s):=\lim_{n\to\infty}V_n(s)$.

We first show that $V$ is bounded. Since $(V_n)$ is Cauchy, choose $N$ such that $\|V_n-V_N\|_\infty\le1$ for all $n\ge N$. Then for all $s$ and all $n\ge N$,
\[
|V_n(s)|\le |V_N(s)|+1\le \|V_N\|_\infty+1.
\]
Taking $n\to\infty$ gives $|V(s)|\le\|V_N\|_\infty+1$, so $V$ is bounded.

We next show measurability. A pointwise limit of measurable real-valued functions is measurable: for example,
\[
\{s:V(s)<a\}=\bigcup_{q<a,\ q\in\mathbb{Q}}\bigcup_{N=1}^{\infty}\bigcap_{n\ge N}\{s:V_n(s)<q\},
\]
which is measurable because it is built from countable unions and intersections of measurable sets. Hence $V\in\mathcal{B}(S)$.

Finally, fix $\varepsilon>0$ and choose $N$ such that $\|V_n-V_m\|_\infty<\varepsilon$ for all $m,n\ge N$. For $n\ge N$ and every $s$,
\[
|V_n(s)-V(s)|=\lim_{m\to\infty}|V_n(s)-V_m(s)|\le\varepsilon.
\]
Taking the supremum over $s$ gives $\|V_n-V\|_\infty\le\varepsilon$ for all $n\ge N$. Thus $V_n\to V$ in sup norm, so $\mathcal{B}(S)$ is complete.
\end{proof}

\subsection{Auxiliary measurability and total-variation facts}
\begin{lemma}[Kernel integration]
\label{lem:app-kernel-integration}
Let $k:X\to\Delta(Y)$ be a Markov kernel. If $f:Y\to\R$ is bounded and measurable, then
\[
x\mapsto \int_Y f(y)k(dy\mid x)
\]
is measurable and bounded by $\|f\|_\infty$ in absolute value.
\end{lemma}

\begin{proof}
For $f=\mathbf{1}_B$, the integral is $k(B\mid x)$, measurable by definition of a kernel. By linearity, the claim holds for nonnegative simple functions. For a nonnegative bounded measurable $f$, choose simple $f_n\uparrow f$; the corresponding integrals are measurable and converge pointwise to the integral of $f$ by monotone convergence. A pointwise limit of measurable functions is measurable. For general bounded $f$, write $f=f^+-f^-$. The bound follows from
\[
\left|\int f\,dk(x)\right|\le \int |f|\,dk(x)\le\|f\|_\infty.
\]
\end{proof}

\begin{lemma}[Total-variation test-function bound]
\label{lem:app-tv-bound}
With the convention
\[
\TV(\mu,\nu)=\sup_{\|f\|_\infty\le1}\left|\int f\,d\mu-\int f\,d\nu\right|,
\]
for every bounded measurable $f$,
\[
\left|\int f\,d\mu-\int f\,d\nu\right|\le \|f\|_\infty\TV(\mu,\nu).
\]
\end{lemma}

\begin{proof}
If $\|f\|_\infty=0$, the claim is immediate. Otherwise define $g=f/\|f\|_\infty$. Then $\|g\|_\infty\le1$, so by the definition of $\TV$,
\[
\left|\int g\,d\mu-\int g\,d\nu\right|\le\TV(\mu,\nu).
\]
Multiplying by $\|f\|_\infty$ proves the result.
\end{proof}

\subsection{Proof of Lemma~\ref{lem:bellman-well-defined}}
\begin{proof}
Fix $V\in\mathcal{B}(S)$ and define
\[
f(s',r):=\rho(r)+\gamma V(s').
\]
The map $f$ is measurable because $\rho$ and $V$ are measurable. Under the uniform boundedness hypothesis,
$|\rho(r)|\le R_{\max}$ for all $r\in R$. Hence, for every $(s',r)\in S\otimes R$,
\[
|f(s',r)|\le R_{\max}+\gamma\|V\|_\infty
\]
Thus the integrand is globally bounded and measurable. The integral defining
$(\mathcal{T}_{\mathsf M,\pi}V)(s)$ is therefore finite and absolutely bounded by the same constant.

By Lemma~\ref{lem:app-kernel-integration}, the function $s\mapsto(\mathcal{T}_{\mathsf M,\pi}V)(s)$ is measurable. Taking the supremum over $s$ gives
\[
\|\mathcal{T}_{\mathsf M,\pi}V\|_\infty
\le R_{\max}+\gamma\|V\|_\infty.
\]
Thus $\mathcal{T}_{\mathsf M,\pi}$ maps $\mathcal{B}(S)$ to itself. If $\|V\|_\infty\le V_{\max}:=R_{\max}/(1-\gamma)$, then
\[
R_{\max}+\gamma\|V\|_\infty
\le R_{\max}+\gamma\frac{R_{\max}}{1-\gamma}
=\frac{R_{\max}}{1-\gamma}=V_{\max},
\]
so the closed ball of radius $V_{\max}$ is invariant.
\end{proof}

\subsection{Proof of Lemma~\ref{lem:bellman-monotone}}
\begin{proof}
If $V\le W$ pointwise, then $\rho(r)+\gamma V(s')\le \rho(r)+\gamma W(s')$ for every $(s',r)$ because $\gamma>0$. Integration against the probability kernel $K^\pi(\cdot\mid s)$ preserves order, so $(\mathcal{T}_{\mathsf M,\pi}V)(s)\le(\mathcal{T}_{\mathsf M,\pi}W)(s)$ for every $s$.

For a constant $c$,
\begin{align*}
(\mathcal{T}_{\mathsf M,\pi}(V+c))(s)
&=\int [\rho(r)+\gamma(V(s')+c)]K^\pi(d(s',r)\mid s)\\
&=\int [\rho(r)+\gamma V(s')]K^\pi(d(s',r)\mid s)+\gamma c\int 1\,K^\pi(d(s',r)\mid s)\\
&=(\mathcal{T}_{\mathsf M,\pi}V)(s)+\gamma c.
\end{align*}
Finally,
\begin{align*}
|\mathcal{T}_{\mathsf M,\pi}V(s)-\mathcal{T}_{\mathsf M,\pi}W(s)|
&=\left|\int \gamma(V(s')-W(s'))K^\pi(d(s',r)\mid s)\right|\\
&\le \gamma\int |V(s')-W(s')|K^\pi(d(s',r)\mid s)\\
&\le \gamma\|V-W\|_\infty.
\end{align*}
Taking the supremum over $s$ proves the Lipschitz bound.
\end{proof}

\subsection{Proof of Proposition~\ref{prop:traj-fp}}
\begin{proof}
Let
\[
G_n:=\sum_{t=0}^{n}\gamma^t\rho(r_t)
\]
be the finite-horizon scalarized return generated by the closed-loop kernel $K^\pi$.
By the uniform boundedness assumption,
\[
|G_n|
\le
\sum_{t=0}^{n}\gamma^t |\rho(r_t)|
\le
R_{\max}\sum_{t=0}^{n}\gamma^t
\le
\frac{R_{\max}}{1-\gamma}.
\]
Moreover, for $m>n$,
\[
|G_m-G_n|
\le
R_{\max}\sum_{t=n+1}^{m}\gamma^t
\le
\frac{R_{\max}\gamma^{n+1}}{1-\gamma}.
\]
Hence $(G_n)$ is Cauchy in $L^\infty$ and converges almost surely and in $L^\infty$ to
\[
G:=\sum_{t\ge0}\gamma^t\rho(r_t),
\]
with $|G|\le R_{\max}/(1-\gamma)$ almost surely. Therefore $V^\pi$ is well-defined and
\[
\|V^\pi\|_\infty\le \frac{R_{\max}}{1-\gamma}.
\]

Now split the infinite discounted return into its first reward and tail:
\[
G=\rho(r_0)+\gamma\sum_{t\ge0}\gamma^t\rho(r_{t+1}).
\]
Conditioning on $(s_1,r_0)$ and using the Markov property under $K^\pi$, the conditional expectation of the tail given $s_1$ is $V^\pi(s_1)$. Hence, for every initial state $s$,
\begin{align*}
V^\pi(s)
&=\E[\rho(r_0)+\gamma V^\pi(s_1)\mid s_0=s]\\
&=\int_{S\otimes R}[\rho(r)+\gamma V^\pi(s')]K^\pi(d(s',r)\mid s)\\
&=(\mathcal{T}_{\mathsf M,\pi}V^\pi)(s).
\end{align*}
Thus $V^\pi$ is a fixed point.

By Lemma~\ref{lem:bellman-contraction}, the same operator is a $\gamma$-contraction in the sup norm. Since $\mathcal{B}(S)$ is complete by Lemma~\ref{lem:BS-banach}, Banach's fixed-point theorem implies that the fixed point is unique. Hence the trajectory value and the fixed-point value agree.
\end{proof}

\subsection{Proof of Lemma~\ref{lem:bellman-contraction}}
\begin{proof}
This is the Lipschitz part of Lemma~\ref{lem:bellman-monotone}. For completeness, for all $s$,
\[
|\mathcal{T}_{\mathsf M,\pi}V(s)-\mathcal{T}_{\mathsf M,\pi}W(s)|
\le\gamma\|V-W\|_\infty.
\]
Taking the supremum over $s$ gives
$\|\mathcal{T}_{\mathsf M,\pi}V-\mathcal{T}_{\mathsf M,\pi}W\|_\infty\le\gamma\|V-W\|_\infty$.
Since $0<\gamma<1$, this is a contraction.
\end{proof}

\subsection{Proof of Theorem~\ref{thm:policy-eval-vi}}
\begin{proof}
By Lemma~\ref{lem:bellman-well-defined}, $\mathcal{T}_{\mathsf M,\pi}$ is a self-map of $\mathcal{B}(S)$, and by Lemma~\ref{lem:bellman-contraction} it is a $\gamma$-contraction. Since $\mathcal{B}(S)$ is complete, Banach's theorem gives a unique fixed point $V^\pi$.

For the convergence bound, use induction. Since $V^\pi=\mathcal{T}_{\mathsf M,\pi}V^\pi$,
\[
\|V_{k+1}-V^\pi\|_\infty
=\|\mathcal{T}_{\mathsf M,\pi}V_k-\mathcal{T}_{\mathsf M,\pi}V^\pi\|_\infty
\le\gamma\|V_k-V^\pi\|_\infty.
\]
Iterating this inequality gives
\[
\|V_k-V^\pi\|_\infty\le\gamma^k\|V_0-V^\pi\|_\infty.
\]
\end{proof}

\subsection{Proof of Theorem~\ref{thm:bellman-trace}}
\begin{proof}
The feedback component of $f_\pi$ is the map
\[
V\mapsto \pi_{\mathcal{B}(S)}f_\pi(\star,V)=\mathcal{T}_{\mathsf M,\pi}V.
\]
By Lemma~\ref{lem:bellman-contraction}, this map is $\gamma$-contractive on the complete metric space $\mathcal{B}(S)$. Hence $f_\pi$ is guarded and Definition~\ref{def:banach-trace} applies. Let $z_\star$ be the unique fixed point of the feedback component. Then
\[
z_\star=\mathcal{T}_{\mathsf M,\pi}z_\star.
\]
By Theorem~\ref{thm:policy-eval-vi}, the unique fixed point of $\mathcal{T}_{\mathsf M,\pi}$ is $V^\pi$, so $z_\star=V^\pi$. The traced output is
\[
\Tr^{\mathcal{B}(S)}_{I,\mathcal{B}(S)}(f_\pi)(\star)
=\pi_{\mathcal{B}(S)}f_\pi(\star,z_\star)
=\mathcal{T}_{\mathsf M,\pi}z_\star
=z_\star
=V^\pi.
\]
This is exactly the Bellman fixed point, which equals the trajectory value by Proposition~\ref{prop:traj-fp}.
\end{proof}

\subsection{Proof of Theorem~\ref{thm:compositionality-series}}
\begin{proof}
For $V\in\mathcal{B}(S)$,
\[
(\mathcal{T}_2V)(u)=\int_{S\otimes R_2}[\rho_2(r_2)+\gamma V(s_2)]K_2^{\pi_U}(d(s_2,r_2)\mid u).
\]
Applying $\mathcal{T}_1$ gives
\begin{align*}
(\mathcal{T}_1\mathcal{T}_2V)(s)
&=\int_{U\otimes R_1}\left[\rho_1(r_1)+\gamma(\mathcal{T}_2V)(u_1)\right]K_1^{\pi_S}(d(u_1,r_1)\mid s)\\
&=\int_{U\otimes R_1}\rho_1(r_1)K_1^{\pi_S}(d(u_1,r_1)\mid s)\\
&\quad+\gamma\int_{U\otimes R_1}\int_{S\otimes R_2}[\rho_2(r_2)+\gamma V(s_2)]K_2^{\pi_U}(d(s_2,r_2)\mid u_1)K_1^{\pi_S}(d(u_1,r_1)\mid s).
\end{align*}
This is precisely the expectation generated by first sampling $(u_1,r_1)$ from $K_1^{\pi_S}(\cdot\mid s)$ and then sampling $(s_2,r_2)$ from $K_2^{\pi_U}(\cdot\mid u_1)$:
\[
(\mathcal{T}_{\mathrm{series}}V)(s)=\E[\rho_1(r_1)+\gamma\rho_2(r_2)+\gamma^2V(s_2)\mid s_0=s].
\]
Thus $\mathcal{T}_{\mathrm{series}}=\mathcal{T}_1\circ\mathcal{T}_2$.

For the contraction claim, use the one-step Lipschitz bounds:
\[
\|\mathcal{T}_1\mathcal{T}_2V-\mathcal{T}_1\mathcal{T}_2W\|_\infty
\le \gamma\|\mathcal{T}_2V-\mathcal{T}_2W\|_\infty
\le \gamma^2\|V-W\|_\infty.
\]
\end{proof}

\subsection{Proof of Theorem~\ref{thm:compositionality-parallel}}
\begin{proof}
For a separable continuation $V_1\oplus V_2$, the parallel Bellman backup is
\begin{align*}
&\mathcal{T}_\otimes(V_1\oplus V_2)(s_1,s_2)\\
&=\int [\rho_1(r_1)+\rho_2(r_2)+\gamma V_1(s_1')+\gamma V_2(s_2')]
(K_1^{\pi_1}\otimes K_2^{\pi_2})(d(s_1',r_1),d(s_2',r_2)\mid s_1,s_2).
\end{align*}
Because the kernel is the product kernel and the integrand is a sum of a function of subsystem $1$ and a function of subsystem $2$, the integral splits:
\begin{align*}
\mathcal{T}_\otimes(V_1\oplus V_2)(s_1,s_2)
&=\int[\rho_1(r_1)+\gamma V_1(s_1')]K_1^{\pi_1}(d(s_1',r_1)\mid s_1)\\
&\quad+\int[\rho_2(r_2)+\gamma V_2(s_2')]K_2^{\pi_2}(d(s_2',r_2)\mid s_2)\\
&=(\mathcal{T}_1V_1)(s_1)+(\mathcal{T}_2V_2)(s_2)\\
&=((\mathcal{T}_1V_1)\oplus(\mathcal{T}_2V_2))(s_1,s_2).
\end{align*}
This proves invariance of the separable subspace.

If $V^{\pi_i}$ is the fixed point of $\mathcal{T}_i$, then the identity above gives
\[
\mathcal{T}_\otimes(V^{\pi_1}\oplus V^{\pi_2})=V^{\pi_1}\oplus V^{\pi_2}.
\]
The parallel Bellman operator is a $\gamma$-contraction on $\mathcal{B}(S_1\otimes S_2)$, hence has a unique fixed point. Therefore the unique fixed point is $V^{\pi_1}\oplus V^{\pi_2}$.
\end{proof}

\subsection{Proof of Lemma~\ref{lem:fp-stability}}
\begin{proof}
Since $V^\star=TV^\star$ and ${V^\star}'=T'{V^\star}'$,
\begin{align*}
\|V^\star-{V^\star}'\|_\infty
&=\|TV^\star-T'{V^\star}'\|_\infty\\
&\le \|TV^\star-T{V^\star}'\|_\infty+\|T{V^\star}'-T'{V^\star}'\|_\infty\\
&\le \kappa\|V^\star-{V^\star}'\|_\infty+d_\infty^{(S,M)}(T,T').
\end{align*}
Moving the first term to the left gives
\[
(1-\kappa)\|V^\star-{V^\star}'\|_\infty\le d_\infty^{(S,M)}(T,T'),
\]
and division by $1-\kappa>0$ proves the bound.
\end{proof}

\subsection{Proof of Lemma~\ref{lem:context-recursion}}
\label{app:context-proofs}
\begin{proof}
The proof is by structural induction on the certified linear one-hole context.

For the hole context, plugging in $\mathsf M$ or $\mathsf N$ gives exactly the two primitive transformers. Therefore the gain is $L([\cdot])=1$.

For left composition, compare $T_0\circ T_{\mathsf M}$ and $T_0\circ T_{\mathsf N}$. If $T_0$ has Lipschitz constant $\operatorname{Lip}(T_0)$ on the relevant invariant ball, then
\[
d(T_0T_{\mathsf M},T_0T_{\mathsf N})
\le
\operatorname{Lip}(T_0)d(T_{\mathsf M},T_{\mathsf N}).
\]
Thus the gain is multiplied by at most $\operatorname{Lip}(T_0)$.

For right composition, compare $T_{\mathsf M}\circ T_0$ and $T_{\mathsf N}\circ T_0$. Since $T_0$ maps the source ball into the discrepancy ball on which the hole-side operators are compared,
\[
\sup_V\|T_{\mathsf M}(T_0V)-T_{\mathsf N}(T_0V)\|_\infty
\le
\sup_U\|T_{\mathsf M}U-T_{\mathsf N}U\|_\infty.
\]
Hence right composition by a fixed well-typed transformer does not increase the gain.

For tensor side contexts, equip the product value space with the product sup metric. Tensoring by a fixed side transformer does not introduce a second occurrence of the hole. Therefore the discrepancy between the two plugged tensor expressions is exactly the discrepancy on the component containing the hole, possibly after canonical reindexing. Thus the gain is not increased. Any subsequent additive or nonlinear scalarization must be counted separately as a fixed Lipschitz post-composition, as required by Definition~\ref{def:admitted-linear-context}.

It remains to justify the trace estimates. First consider a single guarded feedback map
\[
F:X\times Z\to Y\times Z,\qquad F=(F_Y,F_Z),
\]
satisfying the Lipschitz assumptions in the statement. For each $x$, let $z_x$ be the unique solution of
\[
z_x=F_Z(x,z_x).
\]
For $x,x'\in X$,
\begin{align*}
d_Z(z_x,z_{x'})
&=d_Z(F_Z(x,z_x),F_Z(x',z_{x'}))\\
&\le d_Z(F_Z(x,z_x),F_Z(x,z_{x'}))
+d_Z(F_Z(x,z_{x'}),F_Z(x',z_{x'}))\\
&\le \alpha_Z d_Z(z_x,z_{x'})
+\eta_Z d_X(x,x').
\end{align*}
Since $\alpha_Z<1$,
\[
d_Z(z_x,z_{x'})
\le
\frac{\eta_Z}{1-\alpha_Z}d_X(x,x').
\]
Using the output estimate,
\begin{align*}
d_Y(\Tr^Z(F)(x),\Tr^Z(F)(x'))
&=d_Y(F_Y(x,z_x),F_Y(x',z_{x'}))\\
&\le d_Y(F_Y(x,z_x),F_Y(x,z_{x'}))
+d_Y(F_Y(x,z_{x'}),F_Y(x',z_{x'}))\\
&\le \beta_Z d_Z(z_x,z_{x'})
+a_X d_X(x,x')\\
&\le
\left(a_X+\frac{\beta_Z\eta_Z}{1-\alpha_Z}\right)d_X(x,x').
\end{align*}
This proves the displayed Lipschitz bound for the traced map.

Now compare two pre-trace maps
\[
F_{\mathsf M},F_{\mathsf N}:X\times Z\to Y\times Z
\]
induced after plugging $\mathsf M$ and $\mathsf N$ into the same certified trace context. Let
\[
\delta:=\sup_{x,z}d_{Y\times Z}\big(F_{\mathsf M}(x,z),F_{\mathsf N}(x,z)\big).
\]
For a fixed $x$, let $z_M$ and $z_N$ be the corresponding feedback fixed points. By the shared uniform guardedness certificate, the $Z$-component of $F_{\mathsf N}$ is $\alpha_Z$-contractive in the feedback variable. Therefore
\begin{align*}
d_Z(z_M,z_N)
&=d_Z(F_{\mathsf M,Z}(x,z_M),F_{\mathsf N,Z}(x,z_N))\\
&\le d_Z(F_{\mathsf M,Z}(x,z_M),F_{\mathsf N,Z}(x,z_M))
+d_Z(F_{\mathsf N,Z}(x,z_M),F_{\mathsf N,Z}(x,z_N))\\
&\le \delta+\alpha_Z d_Z(z_M,z_N).
\end{align*}
Thus
\[
d_Z(z_M,z_N)\le \frac{\delta}{1-\alpha_Z}.
\]
If the output component is $\beta_Z$-Lipschitz in the feedback variable, then
\begin{align*}
d_Y(\Tr F_{\mathsf M}(x),\Tr F_{\mathsf N}(x))
&=d_Y(F_{\mathsf M,Y}(x,z_M),F_{\mathsf N,Y}(x,z_N))\\
&\le \delta+\beta_Z d_Z(z_M,z_N)\\
&\le
\left(1+\frac{\beta_Z}{1-\alpha_Z}\right)\delta.
\end{align*}
Therefore a trace node amplifies the pre-trace discrepancy by at most
\[
1+\frac{\beta_Z}{1-\alpha_Z}.
\]
Combining this estimate with the induction hypothesis for the pre-trace context gives the stated recursive rule for $L(\Tr^Z(\mathcal D))$.

The contraction modulus $\kappa(\mathcal C)$ is obtained by the same structural induction on the certified Lipschitz constants: composition multiplies constants, tensoring takes the maximum under the product sup metric, and a closed trace node is admitted only when its certified external modulus is strictly smaller than one. Hence every admitted closed context has a structurally certified modulus $\kappa(\mathcal C)<1$.
\end{proof}

\subsection{Proof of Theorem~\ref{thm:congruence}}
\begin{proof}
By Lemma~\ref{lem:context-recursion}, plugging $\mathsf M$ and $\mathsf N$ into the same certified admitted linear context gives closed-loop operators satisfying
\[
d_\infty^{(S,M)}(\mathcal{T}_{\mathcal C[\mathsf M]},\mathcal{T}_{\mathcal C[\mathsf N]})
\le
L(\mathcal C)
d_\infty^{X\leftarrow Y}(\mathcal{T}_{\mathsf M},\mathcal{T}_{\mathsf N})
\le
L(\mathcal C)\varepsilon.
\]
The same certificate gives a common contraction modulus $\kappa(\mathcal C)<1$ for the two closed-loop operators. Applying Lemma~\ref{lem:fp-stability} with
$T=\mathcal{T}_{\mathcal C[\mathsf M]}$ and
$T'=\mathcal{T}_{\mathcal C[\mathsf N]}$ yields
\[
\|V_{\mathcal C[\mathsf M]}-V_{\mathcal C[\mathsf N]}\|_\infty
\le
\frac{1}{1-\kappa(\mathcal C)}
d_\infty^{(S,M)}(\mathcal{T}_{\mathcal C[\mathsf M]},\mathcal{T}_{\mathcal C[\mathsf N]})
\le
\frac{L(\mathcal C)}{1-\kappa(\mathcal C)}\varepsilon.
\]
\end{proof}

\subsection{Proof of Lemma~\ref{lem:intertwine}}
\begin{proof}
Fix $s\in S$ and $\widehat V\in\mathcal{B}(\widehat S)$. Since $\pi(a\mid s)=\widehat\pi(a\mid\phi(s))$,
\begin{align*}
(\mathcal{T}^{\pi}(\phi^*\widehat V))(s)
&=\int_A\left[r(s,a)+\gamma\int_S\widehat V(\phi(s'))P(ds'\mid s,a)\right]\pi(da\mid s)\\
&=\int_A\left[\widehat r(\phi(s),a)+\gamma\int_{\widehat S}\widehat V(\widehat s')\widehat P(d\widehat s'\mid \phi(s),a)\right]\widehat\pi(da\mid\phi(s)).
\end{align*}
The first equality is the concrete Bellman backup applied to the pulled-back value. In the second equality, reward preservation gives $r(s,a)=\widehat r(\phi(s),a)$, and transition commutation gives
\[
\int_S \widehat V(\phi(s'))P(ds'\mid s,a)
=\int_{\widehat S}\widehat V(\widehat s')\widehat P(d\widehat s'\mid\phi(s),a).
\]
The final expression is $(\widehat{\mathcal T}^{\widehat\pi}\widehat V)(\phi(s))=(\phi^*\widehat{\mathcal T}^{\widehat\pi}\widehat V)(s)$.
\end{proof}

\subsection{Proof of Theorem~\ref{thm:exact-hom}}
\begin{proof}
Let $\widehat V^{\widehat\pi}$ be the unique fixed point of $\widehat{\mathcal T}^{\widehat\pi}$. By Lemma~\ref{lem:intertwine},
\[
\mathcal{T}^{\pi}(\phi^*\widehat V^{\widehat\pi})
=\phi^*(\widehat{\mathcal T}^{\widehat\pi}\widehat V^{\widehat\pi})
=\phi^*\widehat V^{\widehat\pi}.
\]
Thus $\phi^*\widehat V^{\widehat\pi}$ is a fixed point of the concrete policy-evaluation operator $\mathcal{T}^{\pi}$. The concrete operator is a $\gamma$-contraction, so its fixed point is unique. Therefore $V^\pi=\phi^*\widehat V^{\widehat\pi}$, i.e. $V^\pi(s)=\widehat V^{\widehat\pi}(\phi(s))$.
\end{proof}

\subsection{Proof of Lemma~\ref{lem:opt-intertwine}}
\begin{proof}
Fix $s\in S$. By the homomorphism assumptions, for each action $a$,
\[
r(s,a)+\gamma\int_S \widehat V(\phi(s'))P(ds'\mid s,a)
=\widehat r(\phi(s),a)+\gamma\int_{\widehat S}\widehat V(\widehat s')\widehat P(d\widehat s'\mid\phi(s),a).
\]
Taking the supremum over the same action set $A$ on both sides gives
\begin{align*}
(\mathcal{T}^\star(\phi^*\widehat V))(s)
&=\sup_{a\in A}\left[r(s,a)+\gamma\int_S \widehat V(\phi(s'))P(ds'\mid s,a)\right]\\
&=\sup_{a\in A}\left[\widehat r(\phi(s),a)+\gamma\int_{\widehat S}\widehat V(\widehat s')\widehat P(d\widehat s'\mid\phi(s),a)\right]\\
&=(\widehat{\mathcal T}^\star\widehat V)(\phi(s)).
\end{align*}
This is exactly $\mathcal{T}^\star(\phi^*\widehat V)=\phi^*(\widehat{\mathcal T}^\star\widehat V)$.
\end{proof}

\subsection{Proof of Corollary~\ref{cor:opt-preserve}}
\begin{proof}
Let $\widehat V^\star$ be the unique fixed point of $\widehat{\mathcal T}^\star$. Lemma~\ref{lem:opt-intertwine} gives
\[
\mathcal{T}^\star(\phi^*\widehat V^\star)=\phi^*(\widehat{\mathcal T}^\star\widehat V^\star)=\phi^*\widehat V^\star.
\]
Thus $\phi^*\widehat V^\star$ is a fixed point of $\mathcal{T}^\star$. The optimality operator is a $\gamma$-contraction under Assumption~\ref{ass:opt-wellposed}, so this fixed point is unique. Hence $V^\star=\phi^*\widehat V^\star$.
\end{proof}

\subsection{Proof of Theorem~\ref{thm:approx-hom}}
\begin{proof}
Let $T:=\mathcal{T}^\pi$ and $\widehat T:=\widehat{\mathcal T}^{\widehat\pi}$. For any $\widehat V$ with $\|\widehat V\|_\infty\le V_{\max}$ and any $s\in S$,
\begin{align*}
&\left|T(\phi^*\widehat V)(s)-\phi^*(\widehat T\widehat V)(s)\right|\\
&\le \int_A\big|r(s,a)-\widehat r(\phi(s),a)\big|\pi(da\mid s)\\
&\quad+\gamma\int_A\left|\int_S\widehat V(\phi(s'))P(ds'\mid s,a)-\int_{\widehat S}\widehat V(\widehat s')\widehat P(d\widehat s'\mid\phi(s),a)\right|\pi(da\mid s).
\end{align*}
The reward term is at most $\varepsilon_r$. The transition term compares the expectation of $\widehat V$ under $\phi_\#P(\cdot\mid s,a)$ with its expectation under $\widehat P(\cdot\mid\phi(s),a)$. By Lemma~\ref{lem:app-tv-bound}, it is at most $V_{\max}\varepsilon_P$. Hence
\[
\|T(\phi^*\widehat V)-\phi^*(\widehat T\widehat V)\|_\infty
\le \varepsilon_r+\gamma V_{\max}\varepsilon_P=:\varepsilon.
\]
Apply this with $\widehat V=\widehat V^{\widehat\pi}$. Since $\widehat T\widehat V^{\widehat\pi}=\widehat V^{\widehat\pi}$,
\[
\|T(\phi^*\widehat V^{\widehat\pi})-\phi^*\widehat V^{\widehat\pi}\|_\infty\le\varepsilon.
\]
The concrete fixed point $V^\pi$ satisfies $T V^\pi=V^\pi$. Therefore
\begin{align*}
\|V^\pi-\phi^*\widehat V^{\widehat\pi}\|_\infty
&=\|T V^\pi-\phi^*\widehat V^{\widehat\pi}\|_\infty\\
&\le\|T V^\pi-T(\phi^*\widehat V^{\widehat\pi})\|_\infty
+\|T(\phi^*\widehat V^{\widehat\pi})-\phi^*\widehat V^{\widehat\pi}\|_\infty\\
&\le \gamma\|V^\pi-\phi^*\widehat V^{\widehat\pi}\|_\infty+\varepsilon.
\end{align*}
Rearranging gives
\[
\|V^\pi-\phi^*\widehat V^{\widehat\pi}\|_\infty
\le\frac{\varepsilon_r+\gamma V_{\max}\varepsilon_P}{1-\gamma}.
\]
\end{proof}

\subsection{Proof of Corollary~\ref{cor:op-mismatch}}
\begin{proof}
The first half of the proof of Theorem~\ref{thm:approx-hom} did not use the fixed-point property. It showed directly that for every $\widehat V$ in the ball $\|\widehat V\|_\infty\le V_{\max}$,
\[
\|T(\phi^*\widehat V)-\phi^*(\widehat T\widehat V)\|_\infty
\le \varepsilon_r+\gamma V_{\max}\varepsilon_P.
\]
Taking the supremum over that ball gives the claimed operator mismatch bound.
\end{proof}

\subsection{Proof of Proposition~\ref{prop:abstraction-context}}
\begin{proof}
Corollary~\ref{cor:op-mismatch} states that the adapted concrete block $A_\phi(T)$ and the adapted abstract block $\widehat A_\phi(\widehat T)$ have local typed transformer discrepancy at most
\[
\varepsilon=\varepsilon_r+\gamma V_{\max}\varepsilon_P.
\]
The adapters give both blocks the same external type, so they can be substituted into the same context. Theorem~\ref{thm:congruence} applied to this local discrepancy gives
\[
\|V_{\mathcal C[A_\phi(T)]}-V_{\mathcal C[\widehat A_\phi(\widehat T)]}\|_\infty
\le \frac{L(\mathcal C)}{1-\kappa(\mathcal C)}\varepsilon.
\]
Substituting the expression for $\varepsilon$ proves the proposition.
\end{proof}

\subsection{Proof of Lemma~\ref{lem:kleene}}
\begin{proof}
Let
\[
V_\infty:=\bigvee_{n\ge0}T^n(\bot).
\]
The sequence $T^n(\bot)$ is increasing: $\bot\le T(\bot)$ by minimality of $\bot$, and monotonicity gives $T^n(\bot)\le T^{n+1}(\bot)$ for all $n$. By $\omega$-continuity,
\[
T(V_\infty)=T\left(\bigvee_{n\ge0}T^n(\bot)\right)=\bigvee_{n\ge0}T^{n+1}(\bot).
\]
Because the chain starts at $\bot$, adding or deleting the first term does not change its supremum, so $T(V_\infty)=V_\infty$. Thus $V_\infty$ is a fixed point.

If $W$ is any fixed point of $T$, then $\bot\le W$. By monotonicity, $T^n(\bot)\le T^n(W)=W$ for every $n$. Taking the supremum over $n$ gives $V_\infty\le W$. Therefore $V_\infty$ is the least fixed point.
\end{proof}

\subsection{Proof of Lemma~\ref{lem:prefp-bounds-lfp}}
\begin{proof}
The value space $\cV(S)$ is a complete lattice because $\cQ$ is complete and the order is pointwise. Let
\[
P:=\{D\in\cV(S):T(D)\le D\}
\]
be the set of pre-fixed points of $T$. The hypothesis says $C\in P$. Let $A:=\bigwedge P$, the meet of all pre-fixed points. Since $A\le D$ for every $D\in P$, monotonicity gives $T(A)\le T(D)\le D$ for every $D\in P$. Hence $T(A)\le\bigwedge P=A$, so $A\in P$.

Because $T(A)\le A$ and $T$ is monotone, $T(T(A))\le T(A)$, so $T(A)$ is also a pre-fixed point. Since $A$ is the meet of all pre-fixed points, $A\le T(A)$. Combining $A\le T(A)$ with $T(A)\le A$ gives $T(A)=A$. Thus $A$ is a fixed point. Moreover, every fixed point is a pre-fixed point, so $A\le W$ for every fixed point $W$. Therefore $A=\lfp(T)$.

Since $C\in P$ and $A=\bigwedge P$, we have $\lfp(T)=A\le C$, as claimed.
\end{proof}

\subsection{Proof of Proposition~\ref{prop:concrete-contract-instance}}
\begin{proof}
For $C\in \mathcal V_+(Y)$, the map $x\mapsto\int_Y C(y)K(dy\mid x)$ is measurable by the defining measurability property of Markov kernels and the standard monotone-class argument for nonnegative measurable integrands.  Hence $\mathcal T_{K,c}C\in\mathcal V_+(X)$.

Monotonicity follows from monotonicity of the extended nonnegative integral: if $C\le D$, then
\[
\int C(y)K(dy\mid x)\le \int D(y)K(dy\mid x)
\]
for every $x$, hence $\mathcal T_{K,c}C\le \mathcal T_{K,c}D$.
For $\omega$-continuity, let $C_n\uparrow C$ pointwise with $C_n\in\mathcal V_+(Y)$.  The monotone convergence theorem gives
\[
\int C(y)K(dy\mid x)=\sup_n\int C_n(y)K(dy\mid x),
\]
and therefore $\mathcal T_{K,c}(\bigvee_n C_n)=\bigvee_n\mathcal T_{K,c}(C_n)$.

For series wiring, substitute the definition twice:
\begin{align*}
(\mathcal T_{K_1,c_1}\mathcal T_{K_2,c_2}C)(x)
&=c_1(x)+\gamma\int_Y\left(c_2(y)+\gamma\int_Z C(z)K_2(dz\mid y)\right)K_1(dy\mid x).
\end{align*}
This is exactly the two-micro-step transformer with the second cost and continuation discounted by one additional factor.
For parallel wiring, product kernels and separability give
\begin{align*}
&\mathcal T_{K_1\otimes K_2,\,c_1\otimes_{\cQ}c_2}(C_1\otimes_{\cQ}C_2)(x_1,x_2)\\
&=c_1(x_1)+c_2(x_2)+\gamma\int\big(C_1(y_1)+C_2(y_2)\big)(K_1\otimes K_2)(d(y_1,y_2)\mid x_1,x_2)\\
&=(\mathcal T_{K_1,c_1}C_1)(x_1)+(\mathcal T_{K_2,c_2}C_2)(x_2).
\end{align*}
The feedback clause is Definition~\ref{def:lfp-trace}: a jointly $\omega$-continuous pre-trace transformer is a morphism in $\mathbf{CPO}_{\omega,\bot}$, so its feedback component has a parameterized least fixed point by Kleene iteration, and tracing returns the output component at that fixed point.
\end{proof}

\subsection{Proof of Theorem~\ref{thm:contract-lifting}}
\label{app:contracts-proofs}
\begin{proof}
For series, Assumption~\ref{ass:compilation-laws} gives
\[
\llbracket \mathsf M_2\circ\mathsf M_1\rrbracket(C_Z)
=(\mathcal T_1\circ\mathcal T_2)(C_Z).
\]
The hypotheses give $\mathcal T_2(C_Z)\le C_Y$ and $\mathcal T_1(C_Y)\le C_X$. Since $\mathcal T_1$ is monotone,
\[
(\mathcal T_1\circ\mathcal T_2)(C_Z)
\le \mathcal T_1(C_Y)
\le C_X.
\]
If $X=Z=S$ and $C_X=C_Z=C_S$, the same inequality reads $(\mathcal T_1\circ\mathcal T_2)(C_S)\le C_S$, which is an inductive closed-loop contract. Lemma~\ref{lem:prefp-bounds-lfp} then gives the closed-loop guarantee.

For parallel composition, Assumption~\ref{ass:compilation-laws} gives the tensor transformer. Monotonicity of each $\mathcal T_i$ gives
\[
\mathcal T_i(C_{Y_i})\le C_{X_i}.
\]
The quantale tensor is monotone in each argument because it distributes over joins and is a monoid operation on the complete lattice. Therefore, pointwise on $X_1\otimes X_2$,
\[
(\mathcal T_1\otimes\mathcal T_2)(C_{Y_1}\otimes_{\cQ}C_{Y_2})
=\mathcal T_1(C_{Y_1})\otimes_{\cQ}\mathcal T_2(C_{Y_2})
\le C_{X_1}\otimes_{\cQ}C_{X_2}.
\]

For feedback, fix $C_Y$ and define the feedback functional
\[
g(Z'):=F_Z(C_Y,Z').
\]
By Assumption~\ref{ass:compilation-laws}, $g$ is monotone and $\omega$-continuous, so the traced feedback value is $Z^*:=\lfp(g)$. The hypothesis $F_Z(C_Y,C_Z)\le C_Z$ says exactly that $C_Z$ is a pre-fixed point of $g$. Lemma~\ref{lem:prefp-bounds-lfp} gives $Z^*\le C_Z$. Monotonicity of $F_X$ in the feedback argument yields
\[
\Tr^Z(F)(C_Y)=F_X(C_Y,Z^*)\le F_X(C_Y,C_Z)\le C_X.
\]
Thus the traced circuit satisfies the claimed bound.
\end{proof}

\subsection{Proof of Lemma~\ref{lem:belief-policy}}
\label{app:extensions-proofs}
\begin{proof}
Let $b_0=\nu_0$. For $t\ge1$, let
\[
h_t=(a_0,o_1,\ldots,a_{t-1},o_t)
\]
and let $b_t$ denote the conditional distribution of $s_t$ given $h_t$. Thus $b_t$ is the posterior belief available when action $a_t$ is chosen. Given the original history-dependent policy $\pi$, define a regular conditional distribution
\[
\pi_{B,t}(\cdot\mid b):=\mathbb{P}_\pi(a_t\in\cdot\mid b_t=b),
\]
which exists under the standard-Borel assumptions. This is a Markov kernel from the belief space $B$ to $A$.

We prove by induction on $t$ that the law of $(b_t,a_t)$ under the original POMDP policy equals the law under the belief process controlled by $(\pi_{B,t})$. At $t=0$, both processes start from the same prior belief $b_0=\nu_0$ and choose actions according to the same conditional distribution by construction. Assume the equality holds at time $t$. Given $(b_t,a_t)$, the predictive observation law and Bayes update determine the law of $b_{t+1}$ through the same kernel $P_B$ in both processes. Then $a_{t+1}$ is sampled from $\pi_{B,t+1}(\cdot\mid b_{t+1})$, which by definition matches the conditional law of the original action given the belief. Hence the joint law matches at time $t+1$.

The lifted reward is $r_B(b,a)=\int r(s,a)b(ds)$, so the conditional expected reward given $(b_t,a_t)$ is the same in the original and belief processes. Since the joint laws of $(b_t,a_t)$ match for every fixed $t$, the expected one-step rewards match for every $t$. Boundedness of $r$ and $\gamma<1$ justify exchanging expectation with the absolutely convergent discounted sum. Therefore the discounted expected returns coincide.
\end{proof}

\subsection{Proof of Theorem~\ref{thm:belief-mdp}}
\begin{proof}
By Lemma~\ref{lem:belief-policy}, every history-dependent POMDP policy induces a possibly time-dependent belief policy with the same law of belief-action pairs at each time and the same expected one-step reward at each time. Therefore the expected discounted returns from $\nu_0$ are equal:
\[
V^\pi_{\mathrm{POMDP}}(\nu_0)=V^{(\pi_{B,t})}_{\mathcal M_B}(\nu_0).
\]
Conversely, any belief policy is a valid POMDP history policy because the belief $b_t$ is a measurable function of the observation-action history. Thus the two policy classes generate the same achievable return values when expressed at the belief state. Taking suprema over the corresponding policy classes gives
\[
V^\star_{\mathrm{POMDP}}(\nu_0)=V^\star_{\mathcal M_B,\mathrm{nonstat}}(\nu_0).
\]
For a stationary belief policy, the belief process is a fully observed discounted MDP with state space $B$, transition kernel $P_B$, reward $r_B$, and discount $\gamma$, so the fixed-point semantics from the main text applies directly.
\end{proof}

\subsection{Proof of Lemma~\ref{lem:traj-rn}}
\begin{proof}
For a finite horizon $T$, the prefix density under a policy $\alpha\in\{\mu,\pi\}$ factors as
\[
d\mathbb{P}^T_\alpha
=d\lambda_0(s_0)\prod_{t=0}^{T-1}\alpha(da_t\mid s_t)P(ds_{t+1}\mid s_t,a_t)R_t(dr_t\mid s_t,a_t,s_{t+1}),
\]
where $\lambda_0$ is the common initial distribution and $R_t$ denotes the reward kernel or deterministic reward emission. The transition and reward factors are identical under $\pi$ and $\mu$; only the action kernels differ. Since $\pi(\cdot\mid s)\ll\mu(\cdot\mid s)$, each one-step action density $d\pi(\cdot\mid s)/d\mu(\cdot\mid s)$ exists. Cancelling the common factors gives
\[
\frac{d\mathbb{P}^{T}_\pi}{d\mathbb{P}^{T}_\mu}
=\prod_{t=0}^{T-1}\frac{d\pi(\cdot\mid s_t)}{d\mu(\cdot\mid s_t)}(a_t).
\]
When the action kernels admit densities with respect to a common reference measure, the derivative can be written as the familiar product $\prod_{t=0}^{T-1}\pi(a_t\mid s_t)/\mu(a_t\mid s_t)$. 

For the infinite-horizon statement, the finite-prefix likelihood ratios form a nonnegative martingale under $\mathbb{P}_\mu$ with mean one. If the infinite-horizon Radon--Nikodym derivative exists, it is the almost-sure limit of this martingale by the standard martingale convergence characterization of projective-limit likelihood ratios.
\end{proof}

\subsection{Proof of Proposition ~\ref{prop:ope-factor}}
\begin{proof}
In the parallel case, the finite-prefix behavior and target measures factorize as products over modules:
\[
\mathbb{P}^{T}_\mu=\bigotimes_j\mathbb{P}^{T,(j)}_\mu,
\qquad
\mathbb{P}^{T}_\pi=\bigotimes_j\mathbb{P}^{T,(j)}_\pi.
\]
For product measures, Radon--Nikodym derivatives multiply. Therefore
\[
\frac{d\mathbb{P}^{T}_\pi}{d\mathbb{P}^{T}_\mu}
=\prod_j\frac{d\mathbb{P}^{T,(j)}_\pi}{d\mathbb{P}^{T,(j)}_\mu}.
\]

In the series case, the assumed interface variable $U_t$ gives a chain-rule decomposition of the prefix measure into conditional subtrajectory laws. If
\[
d\mathbb{P}^{T}_\alpha=dQ^{(1)}_\alpha\,dQ^{(2)}_\alpha(\cdot\mid U)\cdots dQ^{(m)}_\alpha(\cdot\mid U)
\]
for $\alpha\in\{\mu,\pi\}$, then the Radon--Nikodym chain rule gives
\[
\frac{d\mathbb{P}^{T}_\pi}{d\mathbb{P}^{T}_\mu}
=\prod_{j=1}^m\frac{dQ^{(j)}_\pi}{dQ^{(j)}_\mu},
\]
with each factor interpreted conditionally on the relevant interface variables. This is the stated module-wise factorization.
\end{proof}

\subsection{Proof of Corollary~\ref{cor:ope-second-moment}}
\begin{proof}
Let
\[
W_j:=\frac{d\mathbb{P}^{T,(j)}_\pi}{d\mathbb{P}^{T,(j)}_\mu}
\]
be the module-wise finite-prefix importance weights. Proposition~\ref{prop:ope-factor} gives $W=\prod_j W_j$ for the global weight. If the $W_j$ are independent under $\mathbb{P}^T_\mu$, then
\[
\E_\mu[W^2]
=\E_\mu\left[\prod_j W_j^2\right]
=\prod_j\E_\mu[W_j^2].
\]
Taking logarithms gives
\[
\log\E_\mu[W^2]=\sum_j\log\E_\mu[W_j^2],
\]
which is the claimed identity.
\end{proof}

\subsection{Proof of Theorem~\ref{thm:tracking-functor}}
\begin{proof}
For each $t$, apply Lemma~\ref{lem:fp-stability} to $T=\mathcal T_t$ and $T'=\mathcal T_{t+1}$. Both are $\gamma$-contractions and their fixed points are $V_t^\star$ and $V_{t+1}^\star$. The operator discrepancy is exactly
\[
\eta_t=\sup_{V\in\mathbb B}\|\mathcal T_{t+1}V-\mathcal T_tV\|_\infty.
\]
Thus
\[
\|V_{t+1}^\star-V_t^\star\|_\infty\le\frac{\eta_t}{1-\gamma}.
\]
For the cumulative bound, use the triangle inequality:
\[
\|V_T^\star-V_0^\star\|_\infty
\le\sum_{t=0}^{T-1}\|V_{t+1}^\star-V_t^\star\|_\infty
\le\sum_{t<T}\frac{\eta_t}{1-\gamma}.
\]
\end{proof}

\subsection{Proof of Corollary~\ref{cor:iterate-tracking}}
\begin{proof}
By definition, $V_{t+1}=\mathcal T_tV_t$, while $V_t^\star=\mathcal T_tV_t^\star$. Therefore
\[
\|V_{t+1}-V_t^\star\|_\infty
=\|\mathcal T_tV_t-\mathcal T_tV_t^\star\|_\infty
\le\gamma\|V_t-V_t^\star\|_\infty.
\]
Then
\begin{align*}
\|V_{t+1}-V_{t+1}^\star\|_\infty
&\le\|V_{t+1}-V_t^\star\|_\infty+\|V_t^\star-V_{t+1}^\star\|_\infty\\
&\le\gamma\|V_t-V_t^\star\|_\infty+\frac{\eta_t}{1-\gamma},
\end{align*}
where the second term uses Theorem~\ref{thm:tracking-functor}. Iterating this scalar recursion gives
\[
\|V_T-V_T^\star\|_\infty
\le \gamma^T\|V_0-V_0^\star\|_\infty+
\sum_{t<T}\gamma^{T-1-t}\frac{\eta_t}{1-\gamma}.
\]
\end{proof}

\subsection{Proof of Proposition~\ref{prop:parallel-factor}}
\label{app:example-proofs}
\begin{proof}
The proof is the trajectory version of Theorem~\ref{thm:compositionality-parallel}. Under product dynamics and product policy, the two coordinate processes are independent conditional on their initial states. Therefore
\begin{align*}
V^\pi(s_1,s_2)
&=\E\left[\sum_{t\ge0}\gamma^t(r_1(s_{1,t},a_{1,t})+r_2(s_{2,t},a_{2,t}))\mid s_{1,0}=s_1,s_{2,0}=s_2\right]\\
&=\E\left[\sum_{t\ge0}\gamma^t r_1(s_{1,t},a_{1,t})\mid s_{1,0}=s_1\right]
+\E\left[\sum_{t\ge0}\gamma^t r_2(s_{2,t},a_{2,t})\mid s_{2,0}=s_2\right]\\
&=V^{\pi_1}(s_1)+V^{\pi_2}(s_2).
\end{align*}
Linearity of expectation justifies splitting the sum; bounded rewards justify exchanging expectation and the absolutely convergent discounted series.
\end{proof}

\subsection{Proof of Lemma~\ref{lem:traj-tv-bound}}
\begin{proof}
For any $V$ with $\|V\|_\infty\le V_{\max}$ and any state-action input $(x,a)$,
\begin{align*}
&\left|\tilde r(x,a)+\gamma\int V(y)\widetilde P(dy\mid x,a)-r(x,a)-\gamma\int V(y)P(dy\mid x,a)\right|\\
&\le |\tilde r(x,a)-r(x,a)|+\gamma\left|\int V\,d\widetilde P-\int V\,dP\right|\\
&\le \varepsilon_r+\gamma V_{\max}\varepsilon_P,
\end{align*}
where the last step uses Lemma~\ref{lem:app-tv-bound}. Integrating over the fixed policy kernel cannot increase the uniform bound. Taking the supremum over states and over $V\in\mathbb B$ gives
\[
d_\infty(T,\widetilde T)\le\varepsilon_r+\gamma V_{\max}\varepsilon_P.
\]
\end{proof}

\subsection{Proof of Lemma~\ref{lem:series-mismatch}}
\begin{proof}
For any $V\in\mathbb B_S$,
\begin{align*}
\|T_1T_2V-\widetilde T_1\widetilde T_2V\|_\infty
&\le\|T_1T_2V-\widetilde T_1T_2V\|_\infty
+\|\widetilde T_1T_2V-\widetilde T_1\widetilde T_2V\|_\infty\\
&\le \epsilon_1+\gamma\|T_2V-\widetilde T_2V\|_\infty\\
&\le \epsilon_1+\gamma\epsilon_2.
\end{align*}
The first term is bounded by the definition of $\epsilon_1$ on the relevant ball. The second term uses the $\gamma$-Lipschitz property of $\widetilde T_1$ and the definition of $\epsilon_2$. Taking the supremum over $V\in\mathbb B_S$ proves the result.
\end{proof}

\subsection{Proof of Lemma~\ref{lem:two-module-robust}}
\begin{proof}
By Lemma~\ref{lem:traj-tv-bound}, the local module mismatches satisfy
\[
\epsilon_i\le \varepsilon_r^{(i)}+\gamma V_{\max}\varepsilon_P^{(i)}.
\]
By Lemma~\ref{lem:series-mismatch}, the two-step macro-operator mismatch satisfies
\[
d_\infty(T_{\mathrm{ser}},\widetilde T_{\mathrm{ser}})
\le\epsilon_1+\gamma\epsilon_2.
\]
Both macro operators are $\gamma^2$-contractions, because each is a composition of two $\gamma$-Lipschitz micro-step transformers. Lemma~\ref{lem:fp-stability} with $\kappa=\gamma^2$ therefore yields
\[
\|\widetilde V^\pi-V^\pi\|_\infty
\le\frac{\epsilon_1+\gamma\epsilon_2}{1-\gamma^2}.
\]
Substituting the explicit local bounds gives the stated inequality.
\end{proof}

\end{document}